\documentclass{article}

\PassOptionsToPackage{numbers, compress}{natbib}
\usepackage{hyperref}       % hyperlinks
\usepackage{url}            % simple URL typesetting
\usepackage{booktabs}       % professional-quality tables
\usepackage{amsfonts}       % blackboard math symbols
\usepackage{nicefrac}       % compact symbols for 1/2, etc.
\usepackage{microtype}      % microtypography
\usepackage{xcolor}         % colors
\usepackage{algorithm}
\usepackage{algorithmic}
\usepackage{wrapfig}
\usepackage{amsmath, amsthm, mathtools, bbm}
\usepackage{graphicx}
\usepackage{subcaption}
\usepackage{enumitem}

\usepackage[accepted]{icml2023}

\usepackage[capitalize,noabbrev]{cleveref}

\theoremstyle{plain}

\theoremstyle{definition}

\theoremstyle{remark}

\newcommand{\cM}{\mathcal{M}}
\newcommand{\cS}{\mathcal{S}}
\newcommand{\cA}{\mathcal{A}}
\newcommand{\cR}{\mathcal{R}}
\newcommand{\cP}{\mathcal{P}}
\newcommand{\cJ}{\mathcal{J}}
\newcommand{\iirchi}[2]{\raisebox{0.6\depth}{$#1\chi$}}
\DeclareRobustCommand{\lchi}{{\mathpalette\iirchi\relax}}

\newtheorem{Proposition}{Proposition}
\newtheorem{innercustomgeneric}{\customgenericname}
\providecommand{\customgenericname}{}
\newcommand{\newcustomtheorem}[2]{%
  \newenvironment{#1}[1]
  {%
   \renewcommand\customgenericname{#2}%
   \renewcommand\theinnercustomgeneric{##1}%
   \innercustomgeneric
  }
  {\endinnercustomgeneric}
}
\newcustomtheorem{proposition}{Proposition}

\newcommand{\shaper}{Adversary}
\newcommand{\shapers}{Adversaries}
\newcommand{\learner}{Victim}
\newcommand{\learners}{Victims}

\newcommand{\psname}{Cheap Talk MDP}
\newcommand{\fname}{message}

\DeclareMathAlphabet\mathbfcal{OMS}{cmsy}{b}{n}

\usepackage[textsize=tiny]{todonotes}

\icmltitlerunning{Adversarial Cheap Talk}

\begin{document}

\twocolumn[
\icmltitle{Adversarial Cheap Talk}

% It is OKAY to include author information, even for blind
% submissions: the style file will automatically remove it for you
% unless you've provided the [accepted] option to the icml2023
% package.

% List of affiliations: The first argument should be a (short)
% identifier you will use later to specify author affiliations
% Academic affiliations should list Department, University, City, Region, Country
% Industry affiliations should list Company, City, Region, Country

% Ideally, you should not use this facility. Affiliations will be numbered
% in order of appearance and this is the preferred way.
\icmlsetsymbol{equal}{*}

\begin{icmlauthorlist}
\icmlauthor{Chris Lu}{sch}
\icmlauthor{Timon Willi}{sch}
\icmlauthor{Alistair Letcher}{yyy}
\icmlauthor{Jakob Foerster}{sch}
%\icmlauthor{}{sch}
%\icmlauthor{}{sch}
\end{icmlauthorlist}

\icmlaffiliation{yyy}{aletcher.github.io}
\icmlaffiliation{sch}{FLAIR, University of Oxford}

\icmlcorrespondingauthor{Chris Lu}{christopher.lu@eng.ox.ac.uk}
% \icmlcorrespondingauthor{Firstname2 Lastname2}{first2.last2@www.uk}

% You may provide any keywords that you
% find helpful for describing your paper; these are used to populate
% the "keywords" metadata in the PDF but will not be shown in the document
\icmlkeywords{Machine Learning, ICML}

\vskip 0.3in
]

% this must go after the closing bracket ] following \twocolumn[ ...

% This command actually creates the footnote in the first column
% listing the affiliations and the copyright notice.
% The command takes one argument, which is text to display at the start of the footnote.
% The \icmlEqualContribution command is standard text for equal contribution.
% Remove it (just {}) if you do not need this facility.

\printAffiliationsAndNotice{}  % leave blank if no need to mention equal contribution
% \printAffiliationsAndNotice{\icmlEqualContribution} % otherwise use the standard text.

\begin{abstract}
Adversarial attacks in reinforcement learning (RL) often assume highly-privileged access to the victim's parameters, environment, or data. Instead, this paper proposes a novel adversarial setting called a \textit{\psname} in which an \shaper\ can merely append deterministic messages to the \learner's observation, resulting in a \textit{minimal range of influence}. The \shaper\ cannot occlude ground truth, influence underlying environment dynamics or reward signals, introduce non-stationarity, add stochasticity, see the \learner's actions, or access their parameters. Additionally, we present a simple meta-learning algorithm called Adversarial Cheap Talk (ACT) to train Adversaries in this setting. We demonstrate that an \shaper\ trained with ACT still significantly influences the \learner's \textit{training and testing} performance, despite the highly constrained setting. Affecting train-time performance reveals a new attack vector and provides insight into the success and failure modes of existing RL algorithms. More specifically, we show that an ACT \shaper\ is capable of \textit{harming} performance by interfering with the learner's function approximation, or instead \textit{helping} the \learner's performance by outputting useful features. Finally, we show that an ACT \shaper\ can manipulate messages during train-time to directly and arbitrarily control the \learner\ at test-time. Project video and code are available at \url{https://sites.google.com/view/adversarial-cheap-talk}.
% instantly?
% in a zero-shot manner.
\end{abstract}

\section{Introduction}
% Motivation: What is our threat model
Learning agents are often trained in settings where adversaries are not able to influence underlying environment dynamics or reward signals, but may influence part of the agent's observations. For instance, adversaries can alter and place objects, such as interactive billboards, that will appear in the background of self-driving car datasets. Similarly, adversaries may append arbitrary tags to content that will be used to train recommender systems. In financial markets, adversaries can alter the state of the order-book by submitting orders far from the mid.
%Financial models may be trained to detect correlations in data over which another actor may have limited influence, including making changes to the order-book by submitting orders (essentially for free). 
While these features are `useless' from an information-theoretic point of view, it is common practice in end-to-end deep learning to include them as part of the input and let the model learn which features matter. For instance, \citet{lange2022lottery} showed that many features in common RL environments are unnecessary for solving the task. 
Furthermore, self-driving cars typically do not omit useless parts of the visual input but instead learns to ignore them through training. 
Because they are unimportant, influencing these features often does not require highly-privileged access to the environment. Surprisingly, this paper demonstrates that an actor can still heavily influence the behaviour and performance of learning agents by controlling information \textit{only} in these `useless' channels, \textit{without} knowledge of the agent's parameters or training state. 
%We model these `useless' features using \textit{Cheap Talk} channels.

Most past work in adversarial RL assumes highly-privileged access to the victim. For example, many attacks assume the ability to arbitrarily perturb all of the features in the victim's observation space \citep{huang2017advatt, kos2017delving} or access to a pre-trained victim's parameters \citep{gleave2020advpolicies, wang2022adversarial}. Not only is access to these features often unrealistic in practical settings, they also enable trivial attacks. By perturbing observations, an Adversary can obscure relevant information, such as the ball in a Pong game, rendering the task unsolvable. By accessing a pretrained victim's weights, the Adversary can unsurprisingly generate out-of-distribution inputs that the victim has not observed \citep{gleave2020advpolicies}. Finally, most of these attacks only cause the victim's policy to fail, instead of allowing the adversary to arbitrarily control the victim's behaviour \citep{gu2017badnets, kiourti2020trojdrl, salem2020dynbd, ashcraft2021poisoning, zhang2021optimal}.

We show that an Adversary with extremely limited access can arbitrarily control the performance and behaviour of \textit{learning} agents. To do this, we propose a novel minimum-viable setting called \textit{\psname}s. Adversaries are only allowed to modify `useless' features that are appended to the Victim's observation as a deterministic function of the current state. These features represent parts of an agent's observations that are unrelated to rewards or transition dynamics. In particular, our model applies to Adversaries adding tags to content in recommender systems, or renting space on interactive billboards, or submitting orders far out of the money in financial markets. The setting is \textit{minimal} in that Adversaries cannot use these features to occlude ground truth, influence environment dynamics or reward signals, inject stochasticity, introduce non-stationarity, see the \learner's actions, or access their parameters.

Cheap Talk MDPs are formalised in Section \ref{sec:setting}, and we further justify minimality by proving in Proposition \ref{prop:tabular} that \shapers\ cannot influence \textit{tabular} \learners\ \textit{whatsoever} in Cheap Talk MDPs. In particular, it follows that \shapers\ can only influence \learners\ by interfering with their function approximator.
% We address common questions about the setting in Appendix \ref{app:faq}.
% We also prove more generally in Proposition \ref{prop:optimal} that Adversaries cannot prevent Victims with optimal convergence guarantees to converge to optimal rewards, even in non-tabular settings.
% (They may however slow down their learning process significantly.)

Despite these restrictions, we show that Adversaries can still heavily influence agents parameterised by neural networks. In Section \ref{sec:method}, we introduce a new meta-learning algorithm called Adversarial Cheap Talk (ACT) to train the \shaper. With an extensive set of experiments, we demonstrate in Section \ref{sec:results} that an ACT \shaper\ can manipulate a \learner\ to achieve a number of outcomes:

\begin{enumerate}

\item An ACT \shaper\ can prevent the \learner\ from solving a task, resulting in low rewards \textit{during training}. We provide empirical evidence that the \shaper\ successfully sends messages which induce \textit{catastrophic interference} in the \learner's neural network.

\item Conversely, an ACT \shaper\ can learn to send useful messages that \textit{improve} the \learner's training process, resulting in higher rewards \textit{during training}.

\item Finally, we introduce a training scheme that allows the ACT \shaper\ to directly and arbitrarily control the \learner\ directly \textit{at test-time}.
% , in a \textit{zero-shot} manner.
\end{enumerate}

% We emphasise that these results are surprising given that Adversaries cannot alter transition or reward dynamics, nor influence tabular agents whatsoever.

\section{Related Work}

\subsection{Test-Time Adversarial Attacks}
% Most adversarial attacks on RL only look at test-time, we also look at train-time attacks.
Most work investigating adversarial attacks on deep RL systems focuses on attacks at test-time, i.e. those that attack a fully trained, static policy.  \citet{sun2021strongest}, \citet{huang2017advatt},  \citet{kos2017delving}, \citet{zhang2020robust} and \citet{zhang2021optimal} investigate adversarial attacks that influence test-time performance by directly perturbing the observation space. Unlike in Cheap Talk MDPs, such perturbations can influence the underlying dynamics by obscuring relevant information. \citet{gleave2020advpolicies}, \citet{wang2022adversarial}, and \citet{guo2021adversarial} investigate adversarial attacks that influence test-time performance of reinforcement learning agents that were trained in self-play. In contrast to our method, the adversarial agents can directly interact with the environment and the victim agent. The aforementioned test-time attacks largely work by generating perturbations that push the observations out of the Victim's training distribution. In contrast, in Cheap Talk MDPs, the \learner\ \textit{trains directly with the static adversarial features}; thus, by definition, the \shaper\ cannot generate out-of-distribution or non-stationary inputs. 
% Furthermore, these test-time attacks also assume the ability to directly train against a pre-trained static agent. In contrast, in Cheap Talk MDPs, the adversary is influences a learning agent with random initial parameters.

\subsection{Train-Time Adversarial Attacks}
% There is a lot of reserach on different types of backdoor attacks. What is a backdoor? Earliest ones are on image classifiers and static (fixed patters), newest ones are dynamic (context dependent) and do not directly perturb observations, only interact in environment. In contrast, we don't even interact in environment and have a minimal range of influence.
In contrast to test-time adversarial attacks in RL, in train-time adversarial attacks the Adversary interacts with a \textit{learning} victim. \citet{pinto2017rarl} simultaneously trains an adversary alongside a reinforcement learning agent to robustify the victim's policy. Unlike in this work, the adversary is able to directly apply perturbation forces to the environment and introduce nonstationarity. We make further comparisons in Section \ref{sec:train-time-influence}. Backdoor attacks in reinforcement learning aim to introduce a vulnerability during train-time, which can be triggered at test-time. \citet{kiourti2020trojdrl} and \citet{ashcraft2021poisoning} assume the adversary can directly and fully modify the victim's observations and rewards in order to discretely insert a backdoor that triggers on certain inputs. \citet{wang2021backdoorl} considers the multi-agent setting where the adversary inserts a backdoor using its behaviour in the environment. Unlike in this work, the adversary can influence the underlying environment dynamics. Each of these backdoor attacks simply cause the victim to fail when triggered. In contrast, we use the backdoor to fully control the victim.
%Backdoors can be static, meaning they get activated with fixed patterns, or dynamic, which is when the backdoor gets activated by context-dependent patterns \citep{salem2020dynbd}. For static backdoors, the adversary often directly perturbs the observation space \citep{gu2017badnets, kiourti2020trojdrl, ashcraft2021poisoning}. To introduce dynamic backdoors, the threat model assumes that the adversary has full control over the training process of the agent, giving the adversary the ability to introduce backdoors at train-time \citep{wang2021backdoorl}. In contrast, in our threat model, we assume a minimal range of influence by only appending to the observations. Furthermore, instead of perturbing the observations, \citet{wang2021backdoorl} deploy the adversarial agent directly in the environment; however, in contrast to our work, the deployed agents are able to directly influence the underlying environment. 
%Directly interacting with the environment allows the adversary to introduce non-stationarity and stochasticity. In contrast, our setting does not allow the \shaper\ to introduce either.

\subsection{Failure Modes in Deep Reinforcement Learning}
% idk tbh
Previous works have shown that using neural networks as function approximators in reinforcement learning often results in multiple failure modes due to the non-stationarity of value function bootstrapping \citep{vanHasselt2018deadlytriad}. In particular, works have shown that catastrophic interference \citep{bengio2020td}, capacity loss \citep{lyle2022understanding}, and primacy bias \citep{nikishin2022primacy} often occur, even within a single episode of an environment \citep{fedus2020catastrophic}. \citet{song2020observational} shows that deep reinforcement learning algorithms can often overfit to spurious correlations in the observation space. \citet{shah2022goal} and \citet{di2022goal} show that RL agents can pursue undesirable goals at test time if they correlate with the reward during training. By appending to the observation space, we learn to induce the observational failure modes described in these works.

\subsection{Opponent Shaping / Cheap Talk}
% We extend M-FOS by learning to shape with even less possibilities to shape the agent.
Our method is closely related to the field of opponent shaping. Originally, most opponent shaping algorithms assumed white-box access to their opponents to shape the flow of the opponent's gradient \citep{foerster2018learning, letcher2019differentiable, letcher2019stable, willi2022cola}. Instead, \citet{lu2022mfos} introduce a method to shape opponents without white-box access. However, they still deploy an agent to interact directly in the environment.
%Cheap talk is communication that incurs no cost, is non-binding (it can be ignored and does not limit the agent's action space), and is unverifiable (meaning any information, true or false, can be communicated) \citep{farrell1987cheaptalk}. 
% In reinforcement learning, a cheap talk channel is a part of the state space which can be observed by other agents but does not alter transition dynamics or reward functions. 
Cheap talk channels \citep{crawford1982cheaptalk} in deep reinforcement learning have been used to learn emergent communication \citep{foerster2016l2c} and to solve coordination problems \citep{cao2018negotiation}. To the best of our knowledge, this paper is the first to use a cheap talk channel (and only a cheap talk channel) to shape the behaviour of learning agents.

\section{Background}
\subsection{Reinforcement Learning}
% MDP definitions
A Markov decision process (MDP) consists of a tuple 
$\mathcal{D}=\langle\mathcal{S}, \mathcal { A }, \mathcal{P}, \mathcal { R }, \gamma\rangle$, where $\mathcal{S}$ denotes the state space, $\mathcal{A}$ represents the action space, $\mathcal{P}: \mathcal{S} \times \mathcal{A} \times \mathcal{S} \mapsto [0,1]$ denotes the state transition probability function, $\mathcal{R}: \mathcal{S} \times \mathcal{A} \mapsto \mathbb{R}$ is the reward function and $\gamma \in [0, 1)$ denotes the discount factor. At every timestep $t$, an agent samples an action from its stochastic policy $a_{t} \sim \pi_{\theta}\left(\cdot \mid s_{t}\right)$, where $a_t \in \mathcal{A}$, $s_t \in \mathcal{S}$ and $\theta$ denotes the policy parameterisation. The agent then receives a reward based on the action taken in the current state: $r_{t}=\mathcal{R}\left(s_{t}, a_{t}\right)$. Finally, a new state is sampled according to the transition function $s_{t+1} \sim \mathcal{P}\left(\cdot \mid s_{t}, a_{t}\right)$, resulting in a trajectory $\tau_{\theta}:=\left( \left(s_{0}, a_{0}, r_{0}\right),\left(s_{1}, a_{1}, r_{1}\right), \ldots \right)$. The agent's goal is to maximise its expected discounted return under policy $\pi_\theta$: 
\begin{equation}
\label{eq:obj_func}
J(\theta) = \mathbb{E}_{\pi_\theta}\left[\sum_{t=0}^{\infty}\gamma^{t}r_{t}\right] \,.
\end{equation}

\subsection{Evolution Strategies}
Evolution Strategies \citep[ES]{salimans2017evolutionstrategies} is a derivative-free optimisation method. Let $F : \mathbb{R}^d \rightarrow \mathbb{R}$  be some function we want to optimise over. Instead of optimising $F(\mathbf{x})$ directly, we first blur the objective to
\begin{equation*}
    \mathbb{E}_{\mathbf{\epsilon} \sim N(\mathbf{0}, I_d)}[F(\mathbf{x} + \sigma \mathbf{\epsilon})] \,,
\end{equation*}
where $\sigma$ is a hyper-parameter dictating the amount of Gaussian noise we add. This is useful because
\begin{equation*}
    \nabla_{\mathbf{x}}\mathbb{E}_{\mathbf{\epsilon} \sim N(\mathbf{0}, I_d)}[F(\mathbf{x} + \sigma \mathbf{\epsilon})] = \mathbb{E}_{\mathbf{\epsilon} \sim N(\mathbf{0}, I_d)}\left[\frac{\mathbf{\epsilon}}{\sigma}F(\mathbf{x} + \sigma \mathbf{\epsilon})\right] \,,
\end{equation*}
which allows us to optimise a non-differentiable function using gradient descent techniques in a highly scalable manner. In settings such as Meta-RL, ES allows us to optimise objectives that would be challenging to tackle using meta-gradients. In particular, taking meta-gradients through the entire training trajectory of an RL agent would require taking a meta-gradient through thousands of updates, which is often cumbersome or intractable 
\citep{metz2021gradients}. 

\section{Problem Setting}\label{sec:setting}
% The general setting is that two agents act in an MDP and can talk through a one-direction cheap talk channel and they cannot switch roles.
The setting we propose is of two agents interacting in a \textit{\psname} $\langle\cS, \cA, \cP, \cR, \gamma, \cM, f, \cJ \rangle$, which is effectively an MDP with an augmented state space, whereby features (messages) from a Cheap Talk channel $\cM$ are appended to the states of the original MDP. We refer to the agent which observes the augmented state as the \textit{\learner}, with transition and reward functions $\cP, \cR$ assumed to be independent from $\cM$. Formally, this means that $\cP(\cdot \mid s, m, a) = \cP(\cdot \mid s, m', a)$ and $\cR(s, m, a) = \cR(s, m', a)$ for all messages $m, m' \in \cM$. The agent appending the \fname\ is called the \textit{\shaper}, and is endowed with a deterministic policy $f : \cS \to \cM$ to append messages $m = f(s)$ and an objective function $\cJ$ to optimise (details below).

% What is the \learner's role and what is his objective?
The \learner\ is a standard reinforcement learning agent, selecting actions according to its policy $a_{t} \sim \pi_{\theta}(\cdot \mid s, f(s))$, where $a \in \mathcal{A}, s \in \cS$. The \learner\ optimises its policy $\pi_{\theta}$ with respect to parameters $\theta$ in order to maximise its expected return $J$ as defined in Equation \ref{eq:obj_func}.

% What is the \shaper's role, what can he do in the environment and what are his objectives?
By contrast, the \shaper\ may only act by modifying the cheap talk channel features $f_\phi(s)$ at $s$ at every step, where $f_\phi : \cS \to \cM$ is a deterministic policy (function) of the current state and $\phi$ are the \shaper's parameters. These parameters may only be updated \textit{between} full training/testing episodes of the \learner; the function remains static during episodes to avoid introducing non-stationarity and thus restrict the Adversary's range of influence. The \shaper's objective function $\cJ$ may be picked arbitrarily, and need not be differentiable if it is optimised using ES.

In our train-time experiments we focus on the fully-adversarial setting where objectives are zero-sum, $\cJ = -J$, and the allied setting where objectives are equal, $\cJ =  J$. In test-time experiments we use an entirely different objective, such as reaching for an arbitrary circle in Reacher (see Figure \ref{fig:reacher_goal}). This incentivises the \shaper\ to manipulate the \learner\ into maximising $\cJ$, even if this comes at the cost of the \learner's original objective $J$.

% Assuming there is a notion of test-time starting at timestep $N$, the \shaper's objective corresponds to the \learner's performance under its policy $\pi_\theta$ with respect to the \shaper's reward function:
% \begin{equation}
%     \cJ = \mathbb{E}_{\pi_\theta}\left[\sum_{t=N}^{\infty}\gamma^{t}R^\shapershort\left(s_t, a_t \right)\right] \,.
% \end{equation} 

\subsection{Minimality of Cheap Talk MDPs}\label{subsec:minimality}

To justify our introductory claim that Cheap Talk MDPs only allow for a minimal range of influence, we first prove that \shapers\ cannot influence \learners\ \textit{whatsoever} in the tabular setting, irrespective of the Victim's learning algorithm. It follows that \shapers\ can \textit{only} attack Victims by interfering with their function approximator.

\begin{Proposition}\label{prop:tabular}
In any Cheap Talk MDP, the policy of a \textbf{tabular} \learner\ is independent from its Adversary provided uniform initialisation along $\cM$, namely $\pi_0(\cdot \mid s_i, m_j) = \pi_0(\cdot \mid s_i, m_{j'}) \ \forall \ j, j'$.
\end{Proposition}

\begin{proof}[Proof (Sketch)]
The main intuition is that policy updates for different states do not interfere with each other in tabular settings. Assuming uniform initialisation and noticing that the only states encountered in the environment are of the form $(s, f(s))$, we deduce that the Victim's policy updates are independent from the Adversary's choice of function $f$. Formal proof in Appendix \ref{app:proof1}.
\end{proof}

We also prove more generally that Adversaries cannot prevent Victims with optimal convergence guarantees to converge to optimal rewards, even in non-tabular settings. They may however still harm the Victim by slowing down their convergence rate significantly.

\begin{Proposition}\label{prop:optimal}
A \learner\ which is \textbf{guaranteed to converge to optimal policies in MDPs} will also converge to optimal policies in Cheap Talk MDPs, with an expected return equal to the optimal return for the corresponding no-channel MDP.
\end{Proposition}

\begin{proof}[Proof (Sketch)]
Cheap Talk MDPs are just MDPs with augmented state spaces and transition / reward functions; a Victim will therefore converge regardless. Optimality of the expected return follows from the Bellman equation and independence of transition and reward functions from Adversaries. Formal proof in Appendix \ref{app:proof2}.
\end{proof}

Finally, we further justify minimality in Appendix \ref{app:informal} by elaborating informally on the Adversaries' range of influence.
% introductory claims that \shapers\ cannot (1) occlude the ground truth, (2) influence the environment dynamics / reward functions, (3) see the \learner's actions or parameters, (4) inject stochasticity, or (5) introduce non-stationarity.
We also discuss the possibility of further weakening Cheap Talk MDPs and conclude that all such variations either bring no advantage or reduce to regular MDPs.

\begin{figure*}[ht!]
 \centering
 \begin{subfigure}[]{0.32\linewidth}
     \centering
	\includegraphics[width=\linewidth]{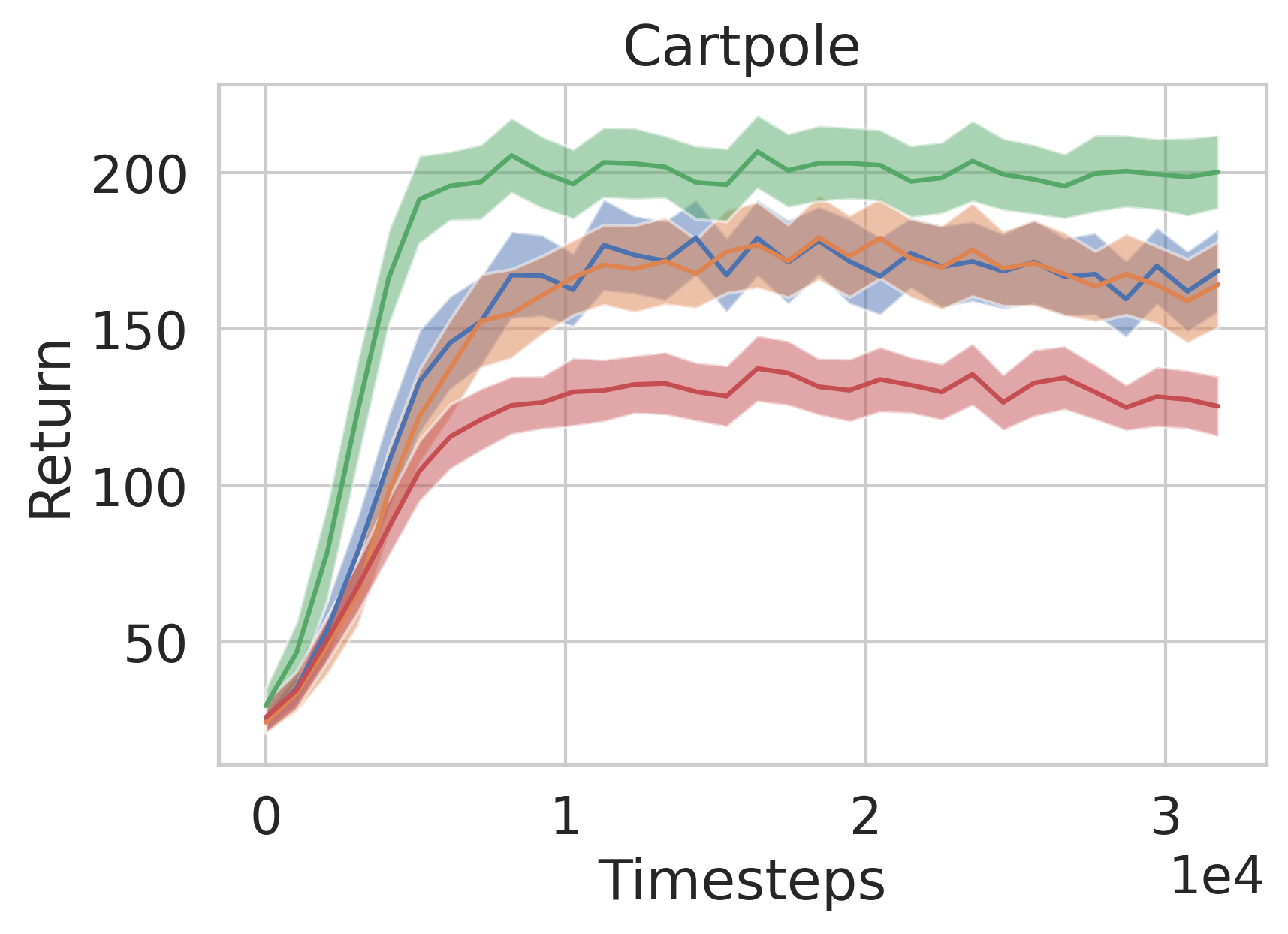}
	\caption{}
	\label{fig:cartpole_train}
 \end{subfigure}
  \begin{subfigure}[]{0.32\linewidth}
     \centering
	\includegraphics[width=\linewidth]{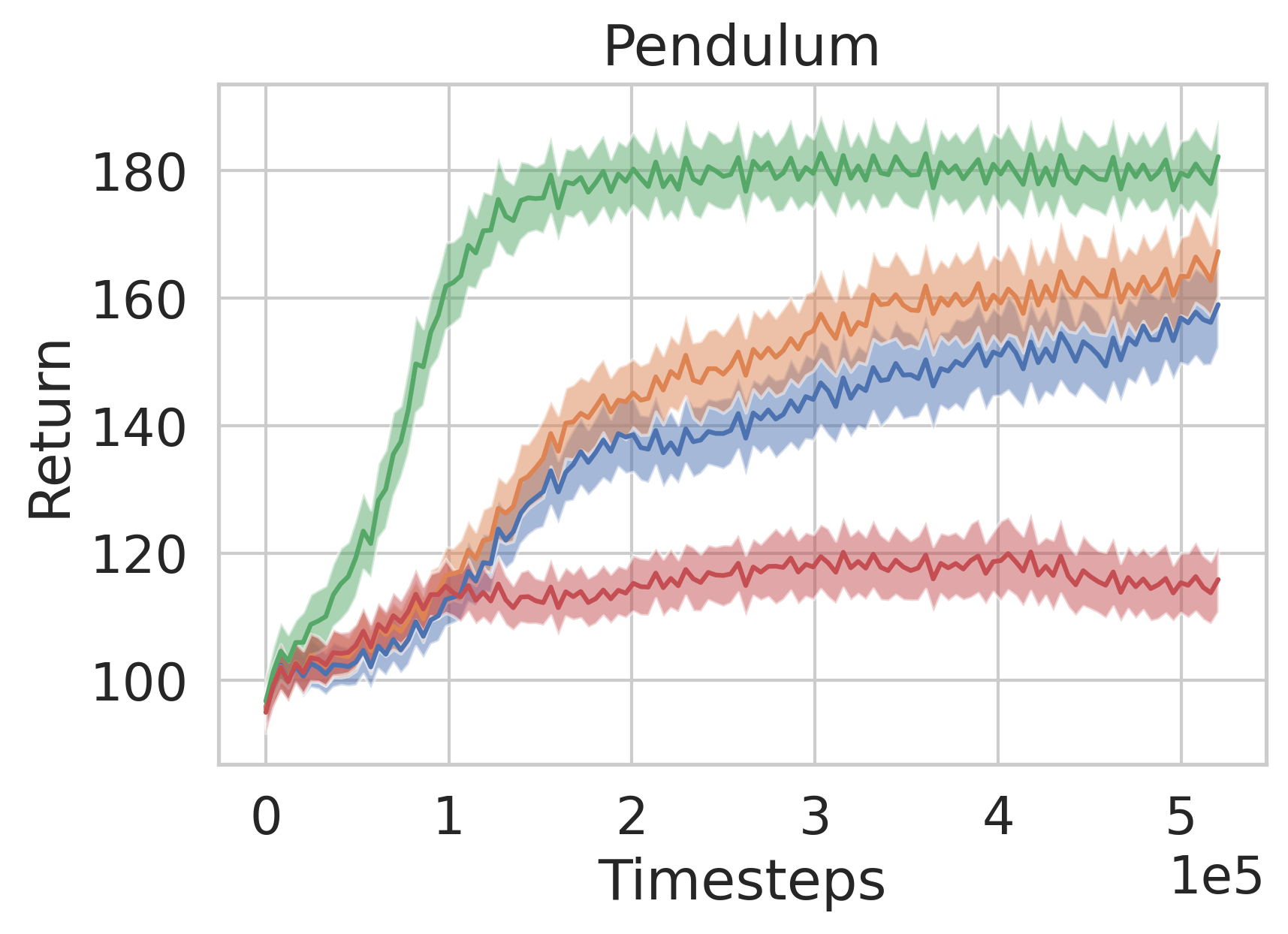}
	\caption{}
	\label{fig:pendulum_train}
 \end{subfigure}
  \begin{subfigure}[]{0.32\linewidth}
     \centering
	\includegraphics[width=\linewidth]{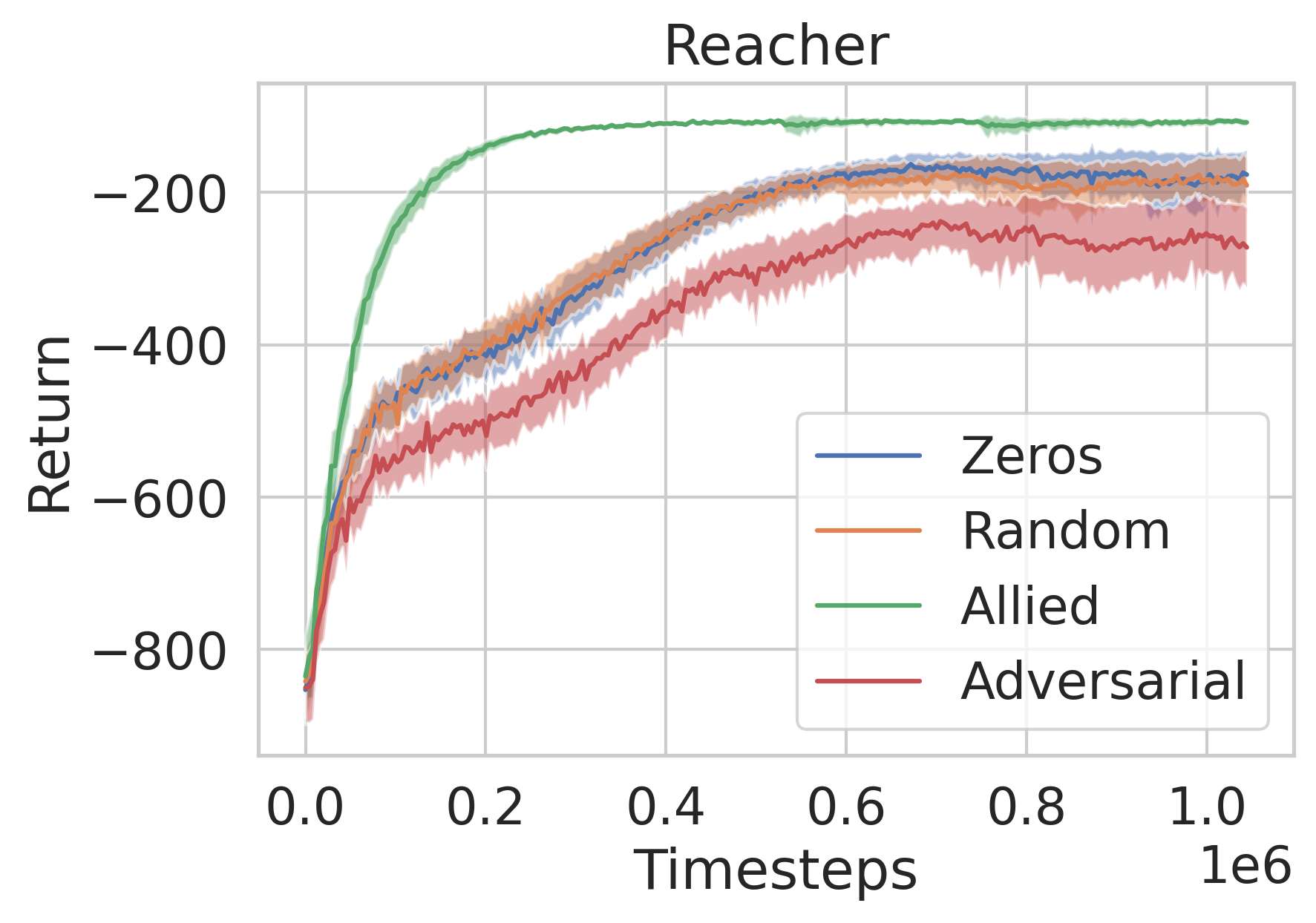}
	\caption{}
	\label{fig:reacher_train}
 \end{subfigure}
\caption{Visualisations of the training curves of the \learner\ across different Adversaries for (a) Cartpole, (b) Pendulum, and (c) Reacher. Error bars denote the standard error across $10$ seeds of Victims trained against a single trained Adversary.}
\label{fig:all_game_trainings}
\end{figure*}

\begin{figure*}[ht]
\hspace{50pt}
 \begin{subfigure}[]{0.32\linewidth}
     \centering
	\includegraphics[width=\linewidth]{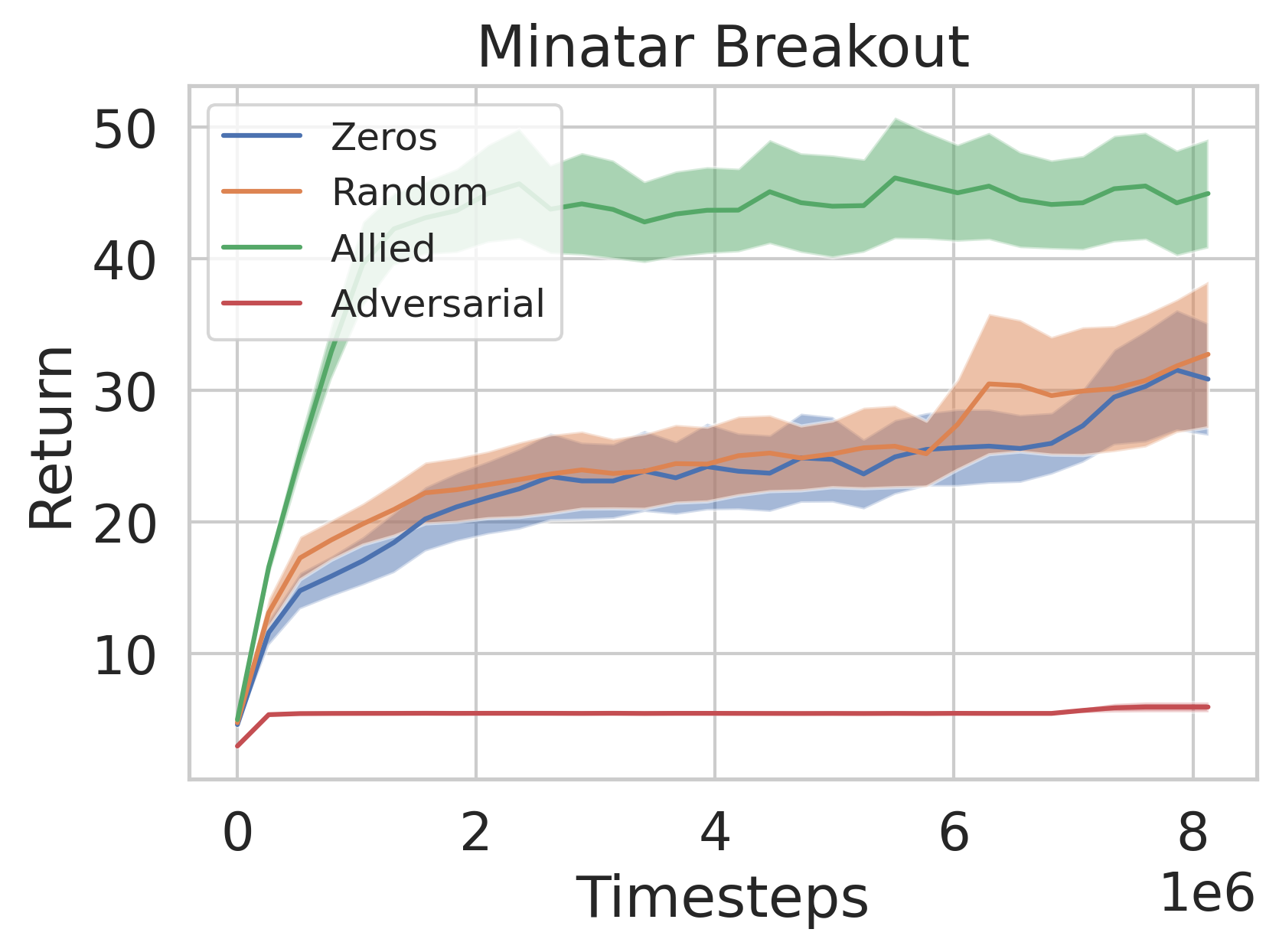}
	\caption{}
	\label{fig:breakout_curves}
 \end{subfigure}\hspace{40pt}
  \begin{subfigure}[]{0.32\linewidth}
     \centering
	\includegraphics[width=\linewidth]{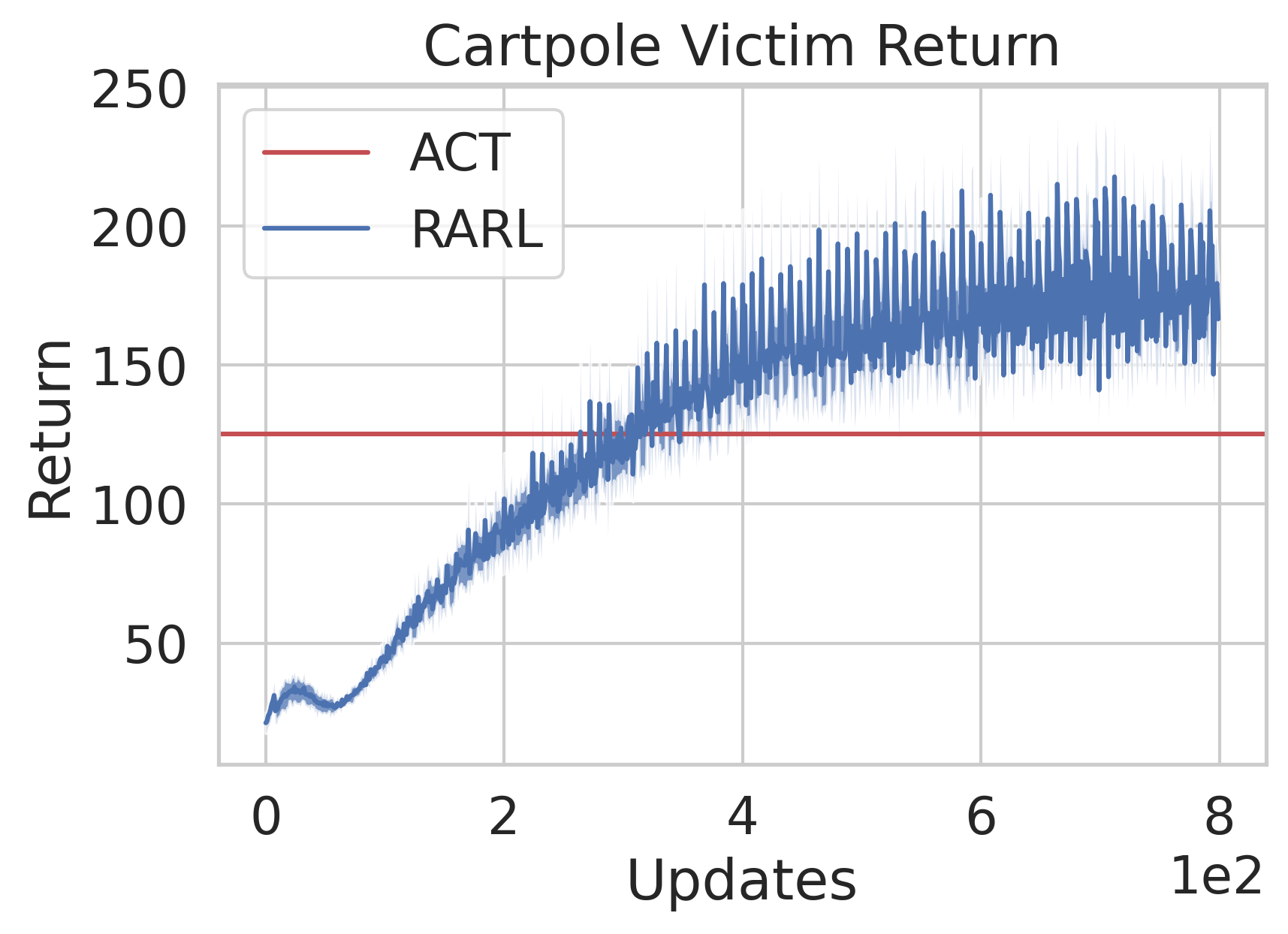}
	\caption{}
	\label{fig:rarl_curves}
 \end{subfigure}
    \caption{(a) Visualisations of the training curves of the Victim in Breakout-Minatar, a higher-dimensional environment, against different Adversaries. (b) Comparing ACT to RARL. Eventually, the Victim learns to overcome the RARL adversary.}
    \label{fig:new_curves}
\end{figure*}

\section{Method}\label{sec:method}

% \begin{wrapfigure}{R}{0.5\textwidth}
\begin{minipage}{0.5\textwidth}
% \vspace{-40pt}
\begin{algorithm}[H]
\caption{Train-time ACT}
\label{alg:mop}
\begin{algorithmic}[1]
\STATE Set $c = \pm 1$ for allied / adversarial
\STATE Initialise \shaper\ parameters $\phi$
\FOR{$m=0$ {\bfseries to} $M$} %\COMMENT{M is number of meta-episodes}
    \STATE Sample $\phi_n \sim \phi + \sigma\epsilon_{n}$ with $\epsilon_n \sim \mathcal{N}(0, I)$
    \FOR{$n=0$ {\bfseries to} $N$}
    \STATE Initialise \learner\ parameters $\theta$
    \STATE rewards = []
    \FOR{$e=0$ {\bfseries to} $E$} %\COMMENT{E is number of training updates}
        \STATE s = env.reset()
        \WHILE{not done} %\COMMENT{T is episode length}
            \STATE $a \sim \pi_\theta(\cdot\mid s, f_{\phi_n}(s))$
            \STATE $r, s$, done = env.step($a$)
            \STATE rewards.append($r$)
        \ENDWHILE
        \STATE Update $\theta$ with PPO to maximise $J$
    \ENDFOR
    \STATE $\cJ_{n} = c \cdot \text{sum(rewards)}/\text{len(rewards)}$
    \ENDFOR
    \STATE Update $\phi$ using ES to maximise $\cJ$
\ENDFOR
\end{algorithmic}
\end{algorithm}
% \vspace{-40pt}
\end{minipage}
% \end{wrapfigure}

\subsection{Meta-Training Procedure}
% Which specific instances of our Problem Setting are we looking at in this paper?
% The two scenarios of the problem setting we consider are influencing train-time and test-time performance respectively.

% How do we solve the instances of the problem settings in general?
Our method, Adversarial Cheap Talk (ACT), treats the problem setting as a meta-learning problem. The \shaper's parameters $\phi$ are only updated after \textit{a full training (and testing) run} of the \learner's parameters $\theta$. In other words, $\phi$ is \textit{static} during the whole training run (inner loop) of $\theta$ and only gets updated once the inner loop completes, which prevents the introduction of non-stationarity. In the outer loop, we optimise the \shaper's objective $\cJ$ with respect to $\phi$ using ES as a black-box optimisation technique. Details, including the Cheap Talk channel sizes and the \learners\ hyperparameters can be found in Appendix \ref{app:hyperparameter-details}. 

\subsection{Train-Time Influence}
% Describe the specific objectives for the \learner and \shaper when influencing the performance during train-time.
% When influencing the agent's performance during train-time, we consider both adversarial and allied settings. Pseudocode is provided in Algorithm \ref{alg:mop}, where $E$ is the number of \learner\ training episodes and $N$ is the ES population size. Letting $c=\pm 1$ for allied and adversarial respectively, the \shaper's objective is $c$ times the \learner's mean reward accumulated over training.
When influencing train-time performance, we set $J$ to be the agent's mean reward throughout its entire training trajectory. We consider both ``Adversarial" and ``Allied'' versions of ACT, whereby Adversaries try to minimise or maximise $J$ respectively ($\cJ = \pm J$). Pseudocode is provided in Algorithm \ref{alg:mop}, where $E$ is the number of \learner\ training episodes and $N$ is the ES population size.

% \begin{equation}
%     \cJ = \frac{c}{E} \sum_{e=0}^E J(\theta) \ .
% \end{equation}
% We look at a specific class of objectives but that in theory we could look at a broader class.
% Note that we could also explore \shaper's objectives that are not necessarily related to the \learner's performance; however, we leave this to future work.

% \subsection{Zero-Shot Test-Time Manipulation}
\subsection{Test-Time Manipulation}
\label{subsec:zero-shot}
% What is Zero-Shot Test-time Manipulation? What is the appropriate objective to capture zero-shot test-time manipulation?
When manipulating test-time behaviour, the goal of the \shaper\ is to use the cheap talk features to maximise some arbitrary objective $\cJ$ during the \learner's test-time; however, the \shaper\ may also communicate messages during the \learner's training. Note that $\cJ$ can be \textit{any} objective, including minimising or maximising the \learner's return. Because the train-time and test-time behaviour of the \shaper\ differ significantly, we parameterise them separately (as $\phi$ and $\psi$ respectively), but optimise them jointly.
%In practice, we introduce a separate \shaper\ for test-time, parameterised by its own set of parameters $\psi$. However, both the train-time \shaper\ $\phi$ and test-time \shaper\ $\psi$ have identical objective function $\cJ$.
% to make the \learner\ susceptible to such manipulation. 

As an example, consider the Reacher environment (see \ref{fig:reacher_goal}), where the \learner\ is trained to control a robot arm to reach for the blue circle. During the \learner's training, the train-time \shaper\ (parameterised by $\phi$) manipulates the cheap talk features to encode spurious correlations in the \learner's policy. At test-time, the test-time \shaper\ (parameterised by $\psi$) manipulates the cheap talk features to take advantage of the spurious correlations and control the \learner\ to have it reach for the yellow circle instead (the \shaper's objective $\cJ$). More concisely, the train-time \shaper\ wants to \textit{create} a backdoor to make the \learner\ susceptible to manipulation at test-time. The test-time \shaper\ wants to \textit{use} this backdoor to control the \learner. The train-time and test-time \shapers\ ($\phi$ and $\psi$) are co-evolved trained end-to-end to maximise $\cJ$. While such optimisation would be difficult for gradient-based methods due to the long-horizon nature of the problem, ES is agnostic to the length of the optimisation horizon.

% In 
% % zero-shot 
% test-time manipulation, the \shaper\ attempts to maximise its objective $\cJ$ during some notion of test-time starting at time $I$. 
% The \shaper's objective is to maximise its average reward over the $I$ testing episodes:
% \begin{equation}
%     \cJ = \frac{1}{IT}\sum_{i=0}^I\sum_{t=0}^{T} R^\shapershort_{i, t} \,.
% \end{equation}
% How are we actually doing it in practice? Why are we doing it like that?
% In practice, we introduce a separate \shaper\ for test-time, parameterised by its own set of parameters $\psi$. However, both the train-time \shaper\ $\phi$ and test-time \shaper\ $\psi$ have identical objective function $\cJ$. We perform this separation to ease learnability because the \shaper\ needs to behave radically differently during train- and test-time. 
%At test-time, the \shaper wants to manipulate the \learner into performing the \shaper's objective. However, if the \shaper attempted this during train-time the \learner would learn to ignore the \shaper. Instead the train-time \shaper needs to gain the trust of the \learner in order to implant a backdoor for use during test-time.
% The train-time \shaper\ wants to \textit{create} a backdoor to make the \learner\ susceptible to manipulation at test-time. The test-time \shaper\ wants to \textit{use} this backdoor to control the \learner.

% Why are we calling it zero-shot? Because the test-time \shaper hasn't interacted with the specific instance of the \learner of that meta-episode before performing against it. We're also formalizing it.
Note that the test-time \shaper\ $\psi$ only gets a single shot to maximise $\cJ$ at the end of the \learner's training and does not have access to (and thus cannot train against) the test-time parameters of the \learner\ $\theta^{\prime}$. To describe this formally, let $\mathcal{T}(\theta \mid \phi)$ denote the distribution induced by the inner loop training with the train-time \shaper\ $\phi$ over \learner\ $\theta$. 
% $\mathcal{X}^{\Phi}: \Phi \times \Theta \rightarrow \Omega(\Theta)$ denote the training process of \learner\ $\theta$ under the train-time \shaper\ $\phi$, which induces a distribution over possible \learner\ parameters $\Omega(\Theta)$ at the end of training, where $\Phi$ and $\Theta$ are the \shaper's and \learner's respective parameter spaces.
Then, in each meta-episode, the test-time \shaper\ $\psi$ interacts with an unseen sample $\theta^{\prime}$ of the distribution over trained \learner s $\theta^{\prime} \sim \mathcal{T}(\cdot \mid \phi)$. In Section \ref{sec:results}, we show that the distribution $\mathcal{T}(\theta \mid \phi)$ has non-trivial variance, suggesting that it is difficult to train against. 
%shows sufficient variance to justify calling the experiments zero-shot. 
Moreover, in Figure \ref{fig:Pendulum_Vis} (Section \ref{sec:results}), we provide empirical evidence that the train-time \shaper\ learns to reduce the variance of $\mathcal{T}(\theta \mid \phi)$ to help the test-time \shaper. Pseudocode is provided in Algorithm \ref{alg:mop_test}, Appendix \ref{app:pseudo}.

\begin{figure*}[t]
 \centering
 \begin{subfigure}[]{0.32\linewidth}
     \centering
	\includegraphics[width=\linewidth]{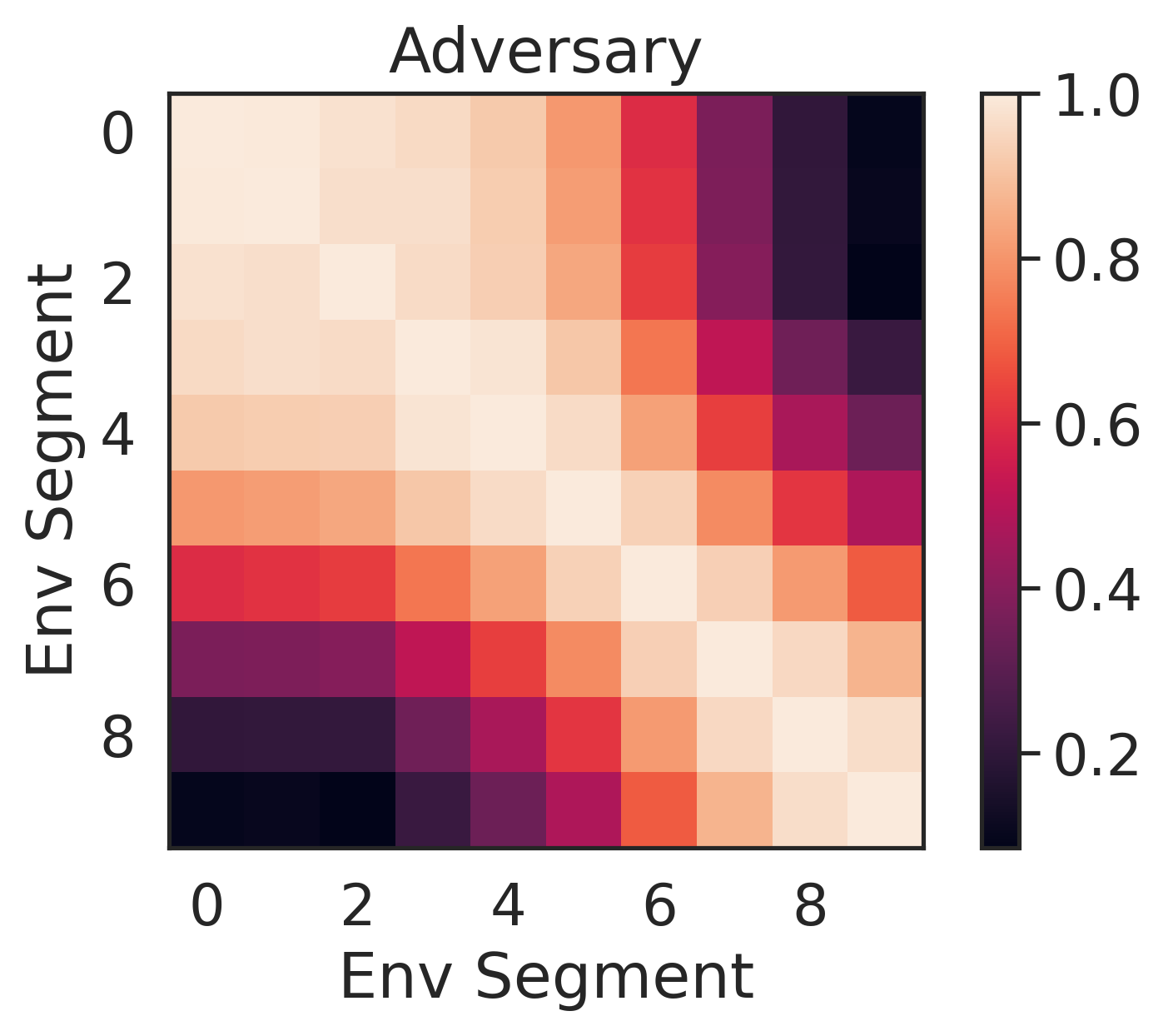}
	\caption{}
	\label{fig:adversary_interference}
 \end{subfigure}
  \begin{subfigure}[]{0.32\linewidth}
     \centering
	\includegraphics[width=\linewidth]{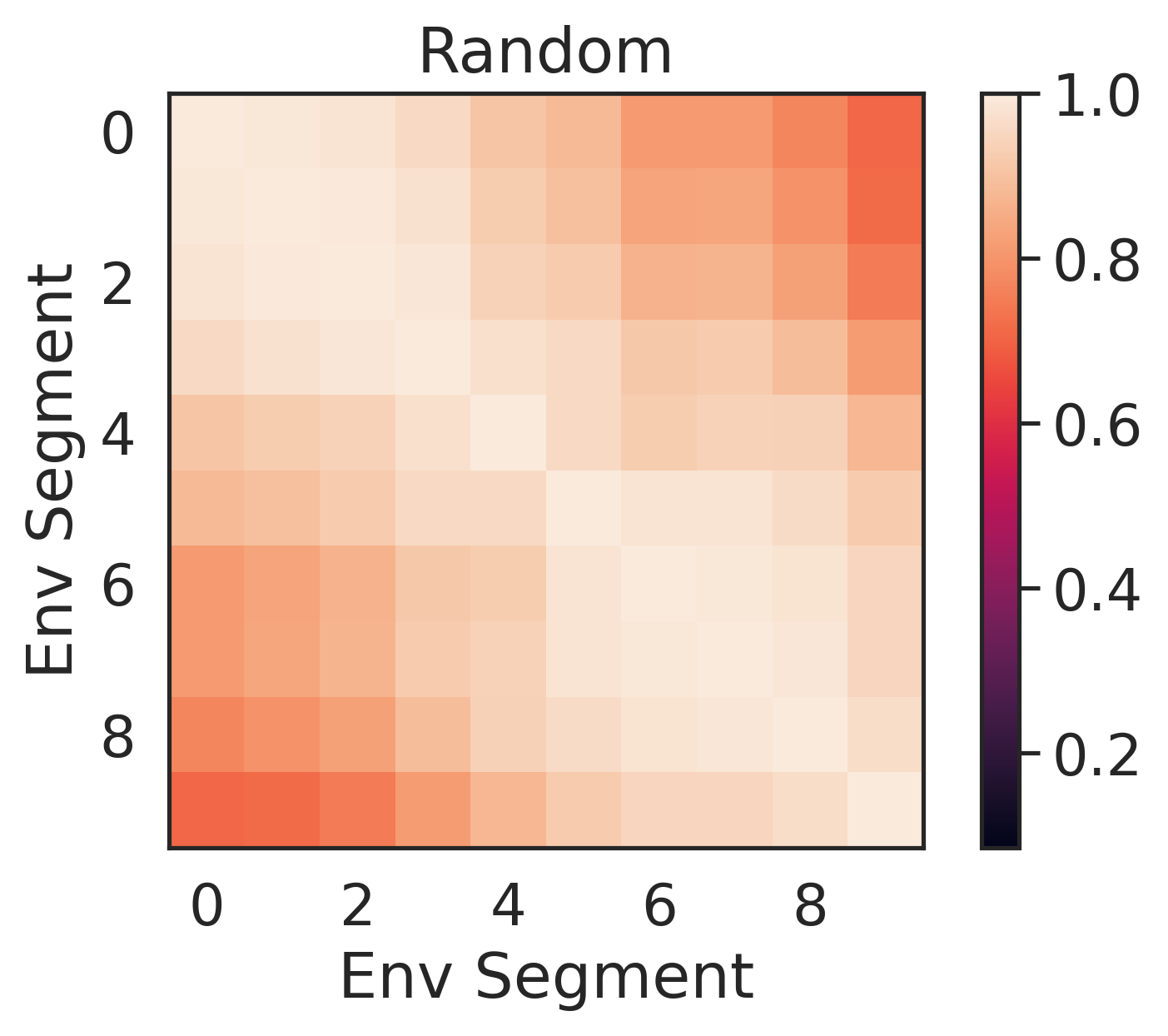}
	\caption{}
	\label{fig:random_interference}
 \end{subfigure}
  \begin{subfigure}[]{0.32\linewidth}
     \centering
	\includegraphics[width=\linewidth]{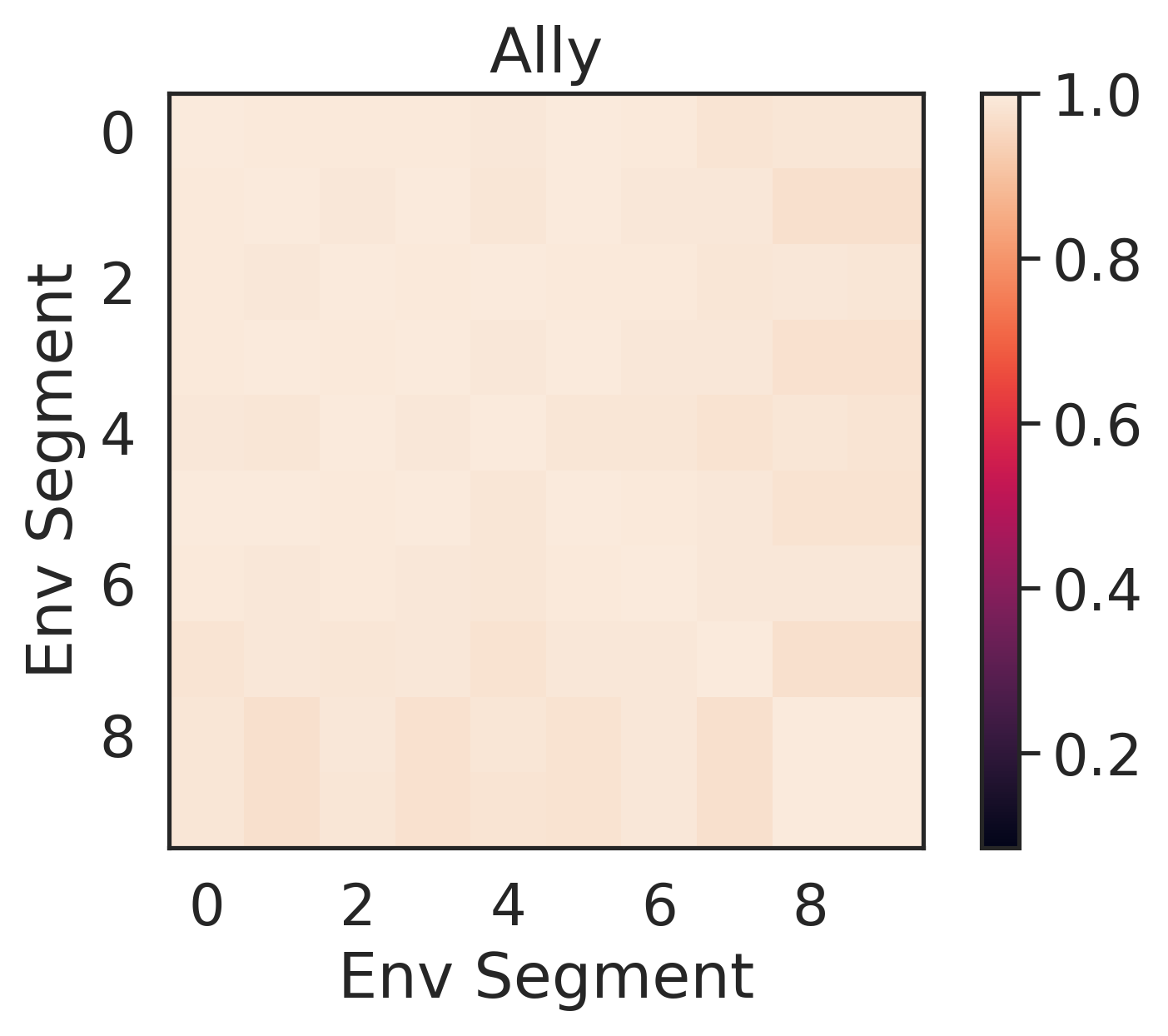}
	\caption{}
	\label{fig:ally_interference}
 \end{subfigure}
    \caption{Visualisations of the cosine distance between gradient updates on different environment segments in cartpole.  We collected each \learner's experience buffer before the agents converge in training and split each into $10$ bins, ordered by the time-step within the environment. We then calculate the gradient update the agents perform on each of these bins. This is the technique used in \citet{fedus2020catastrophic}. For the adversary (a), gradient updates on the early timesteps in an environment interfere with gradient updates on the ending timesteps. For the ally (c), they are positively correlated.}
    \label{fig:Interference_Plots}
\end{figure*}

\section{Experiments and Results}
\label{sec:results}

% Which environments do we use?
We evaluate ACT on three simple gym environments: Cartpole, Pendulum, and Reacher \citep{brockman2016gym}. We also evaluate ACT on Minatar Breakout \cite{young19minatar, gymnax2022github} to test ACT's ability to scale to higher-dimensional environments. The \learner\ is trained with Proximal Policy Optimisation \citep[PPO]{schulman2017ppo}. The \shaper\ is trained using ES \citep{salimans2017evolutionstrategies}.

% Chris is an amazing engineer making deep RL go brrrrr
We train thousands of agents per minute on a single V100 GPU by vectorising both the \textit{PPO algorithm itself} and the environments using Jax \citep{jax2018github}. This allows us to JIT-compile the \textit{full training pipeline} and perform end-to-end deep RL training completely on GPUs. We adapt the environment implementations from \citet{brockman2016gym} and \citet{lenton2021ivy} and use the ES implementation from \citet{evosax2022github}. This compute setup allows us to efficiently perform outer-loop ES on the full training trajectories of inner-loop PPO agents. For example, in Cartpole, we can simultaneously train 8192 PPO agents at a time on a single V100 GPU. Over 1024 generations of ES, this results in training 8,388,608 PPO agents from scratch in 2 hours on 4 V100 GPUs. The longest training time, which was the test-time Reacher setting, took 20 hours to train 1024 generations on 4 V100 GPUs. We include videos of the \learner's performance alongside visualisations of the \shaper's outputs at \href{https://sites.google.com/view/adversarial-cheap-talk/home}{this site.}
% More training details in the Appendix!

Training details are provided in Appendix \ref{app:hyperparameter-details}. We include further ablations where we apply our method to perturbation-based settings and to `useless' features described in \citet{lange2022lottery} instead of Cheap Talk MDPs in Appendix \ref{app:additive-attack}. 
%Note that the PPO implementation uses observation normalisation, so each dimension of the observation has a mean of zero and a standard deviation of one. 

\subsection{Train-Time Influence}
\label{sec:train-time-influence}

Figure \ref{fig:all_game_trainings} and Figure \ref{fig:breakout_curves} show the results of training \learner s alongside four different \shapers.% It is evaluated on four different \shapers:

\begin{enumerate}
\item \textbf{Ally}: meta-trained to \textit{maximise} the \learner's mean reward throughout training.

\item \textbf{\shaper}: meta-trained to \textit{minimise} the \learner's mean reward throughout training.

\item \textbf{Random \shaper}: randomly initialise and fix the \shaper's parameters $\phi$ using LeCun Uniform initialisation \citep{lecun2021uniform}.

\item \textbf{Zeroes \shaper}: appends only zeroes as \fname s.
\end{enumerate}

% General Results: both cooperative and adversarial settings work
% Across all tested environments, the \learner\ trained alongside the Ally outperforms all the others, while the \learner\ trained alongside \shaper\ is outperformed by all the others. 

% The cooperative setting works, which is unsurprising because we can just append helpful features. However (!), we show that we are performing even better than an Oracle!
% Unsurprisingly, the \learner s trained alongside the Allied \shapers\ outperform the baselines:
\paragraph{Ally.} The Ally manages to assist the \learner\ to learn and converge faster -- this is likely done by appending useful features of the environment. We present further analysis in Appendix \ref{app:ablations} Figure \ref{fig:pro_oracle_curves}.
%, we show that this \shaper\ even outperforms an Ally that outputs pre-trained policy logits as \fname s. We hypothesise that the Ally may also be outputting information that helps learn the value function. %over the channel 

% Even though we have the most restricting adversarial setting, the adversarial agent still learns to decrease train-time performance! This is a main result of the paper.
\paragraph{\shaper.} The \learner s trained alongside the \shapers\ are vastly outperformed by the baselines, even though the \shaper\ cannot change the dynamics of the underlying MDP, and cannot add non-stationarity or stochasticity. Moreover, since \shapers\ cannot influence tabular \learners\  by Proposition \ref{prop:tabular}, this must be accomplished through learnt interference with the \learner's function approximator.
% How does the adversarial \shaper learn to affect train-time performance? We provide empirical evidence that it uses catastrophic interference.
We hypothesise that the \shaper\ may be inducing catastrophic interference within the environment, which was observed by \citet{fedus2020catastrophic} in Atari 2600 games. They show that features useful in the early phases of an environment episode can interfere with learning features for performing well in the later phases of an episode. 
%In Figure \ref{fig:Interference_Plots}, we demonstrate that the \shaper\ induces catastrophic interference in both the Adversarial setting by influencing the correlation between gradient updates between different parts of a single inner loop episode. 
We perform the analysis done in \citet{fedus2020catastrophic} in Figure \ref{fig:Interference_Plots} to confirm this hypothesis in the Adversarial setting. Meanwhile, we show that the opposite effect occurs in the Allied setting: the gradient updates are positively correlated, suggesting that the gradient updates at different timesteps aid each other.

We also compare our evolutionary meta-optimisation procedure to Robust Adversarial Reinforcement Learning (RARL) \citep{pinto2017rarl} in Figure \ref{fig:rarl_curves}, which updates the adversary's parameters online using reinforcement learning. In both settings, the Adversary can only communicate over the cheap talk channel; however, RARL is given a larger range of influence. Firstly, RARL introduces non-stationarity since it is updated online during the opponent's learning. Secondly, RARL is parameterised by a stochastic policy, meaning that it can inject stochasticity into the environment. Thirdly, RARL is able to train directly against the \learner's policy online, unlike ACT which cannot view the \learner's policy or actions. However, RARL ultimately underperforms ACT in the adversarial setting since it does not consider the long-term evolution of the \learner's policy. Thus, the \learner\ learns to simply ignore the cheap talk channel. 

% \begin{figure*}[t]
%  \centering
%  \begin{subfigure}[]{0.32\linewidth}
%      \centering
% 	\includegraphics[width=\linewidth]{Figures/Adversary_hm.png}
% 	\caption{}
% 	\label{fig:adversary_interference}
%  \end{subfigure}
%   \begin{subfigure}[]{0.32\linewidth}
%      \centering
% 	\includegraphics[width=\linewidth]{Figures/Random_hm.png}
% 	\caption{}
% 	\label{fig:random_interference}
%  \end{subfigure}
%   \begin{subfigure}[]{0.32\linewidth}
%      \centering
% 	\includegraphics[width=\linewidth]{Figures/Ally_hm.png}
% 	\caption{}
% 	\label{fig:ally_interference}
%  \end{subfigure}
%     \caption{Visualisations of the cosine distance between gradient updates on different environment segments in cartpole.  We collected each \learner's experience buffer before the agents converge in training and split each into $10$ bins, ordered by the time-step within the environment. We then calculate the gradient update the agents perform on each of these bins. This is the technique used in \citet{fedus2020catastrophic}. For the adversary (a), gradient updates on the early timesteps in an environment interfere with gradient updates on the ending timesteps. For the ally (c), they are positively correlated.}
%     \label{fig:Interference_Plots}
% \end{figure*}

\begin{figure*}[ht!]
 \centering
 \begin{subfigure}[]{0.32\linewidth}
     \centering
	\includegraphics[width=\linewidth]{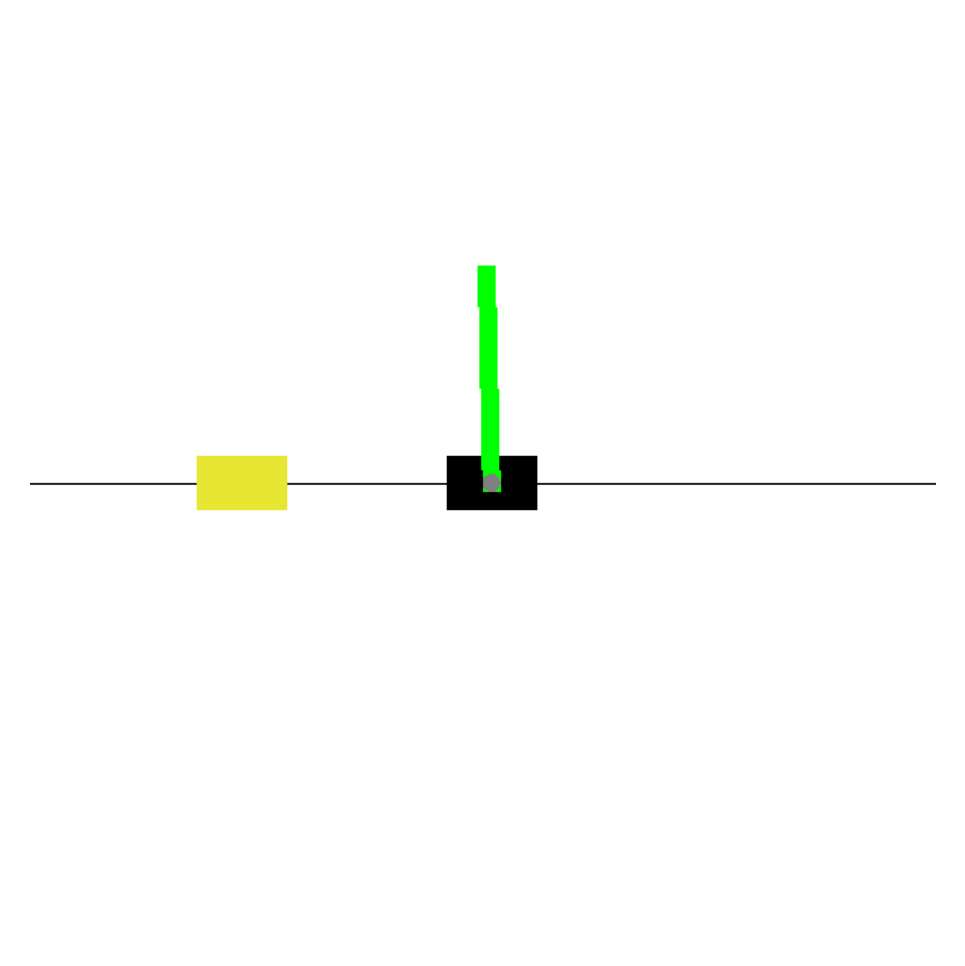}
	\caption{}
	\label{fig:cartpole_goal}
 \end{subfigure}
  \begin{subfigure}[]{0.32\linewidth}
     \centering
	\includegraphics[width=\linewidth]{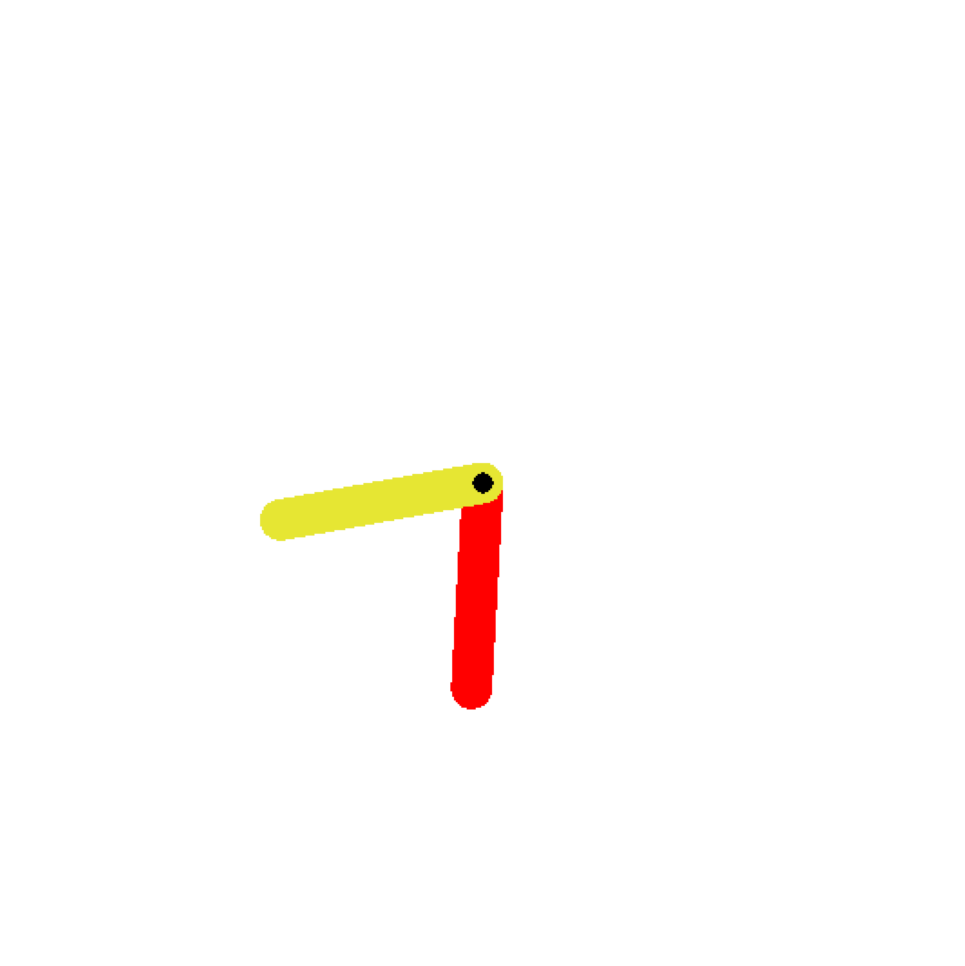}
	\caption{}
	\label{fig:pendulum_goal}
 \end{subfigure}
  \begin{subfigure}[]{0.32\linewidth}
     \centering
	\includegraphics[width=\linewidth]{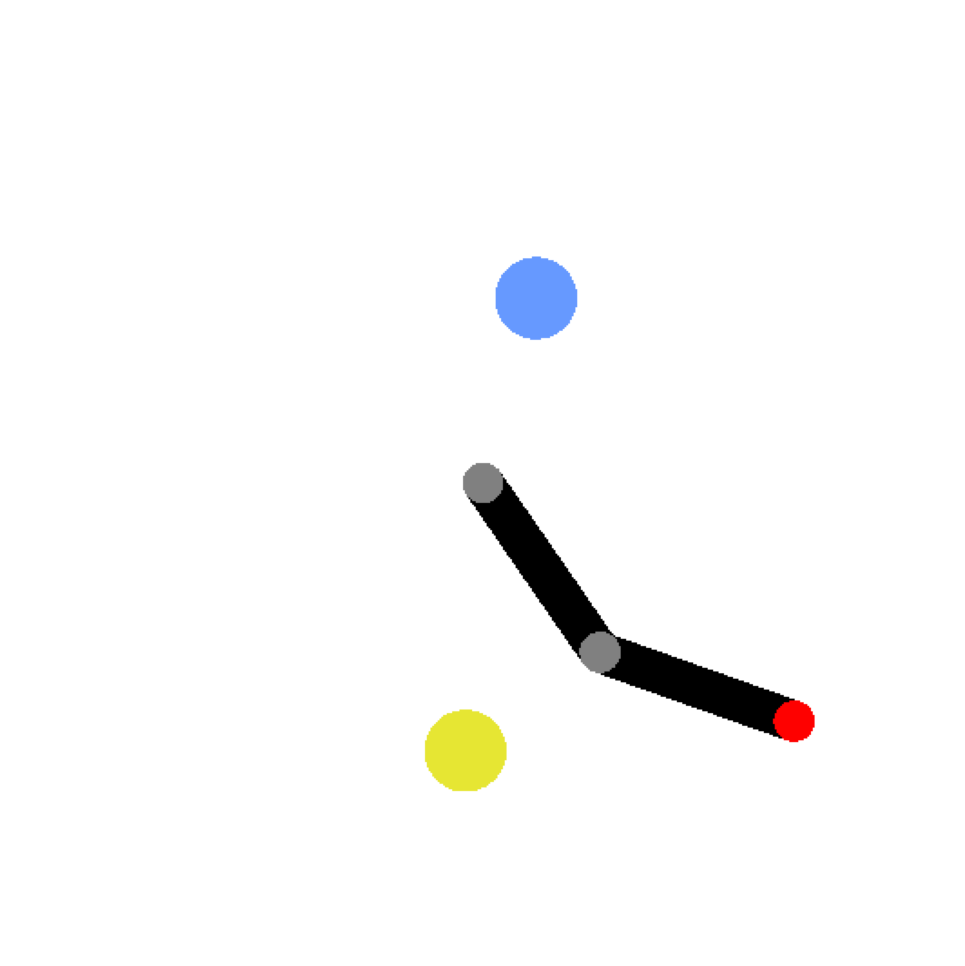}
	\caption{}
	\label{fig:reacher_goal}
 \end{subfigure}
 \caption{Visualisations of our goal-conditioned environments (a) In Cartpole, the \shaper's target is a randomly selected point on the x-axis (the yellow box). (b) In Pendulum, the \shaper's goal is a randomly selected angle (the yellow pole). (c) In Reacher, the  \shaper's goal is a random point, (the yellow circle). The \learner's goal is the blue circle. Videos of this setting can be found \href{https://sites.google.com/view/adversarial-cheap-talk/home}{here}.}
 \label{fig:goal_envs}
\end{figure*}

\begin{figure*}[t]
 \centering
 \begin{subfigure}[]{0.32\linewidth}
     \centering
	\includegraphics[width=\linewidth]{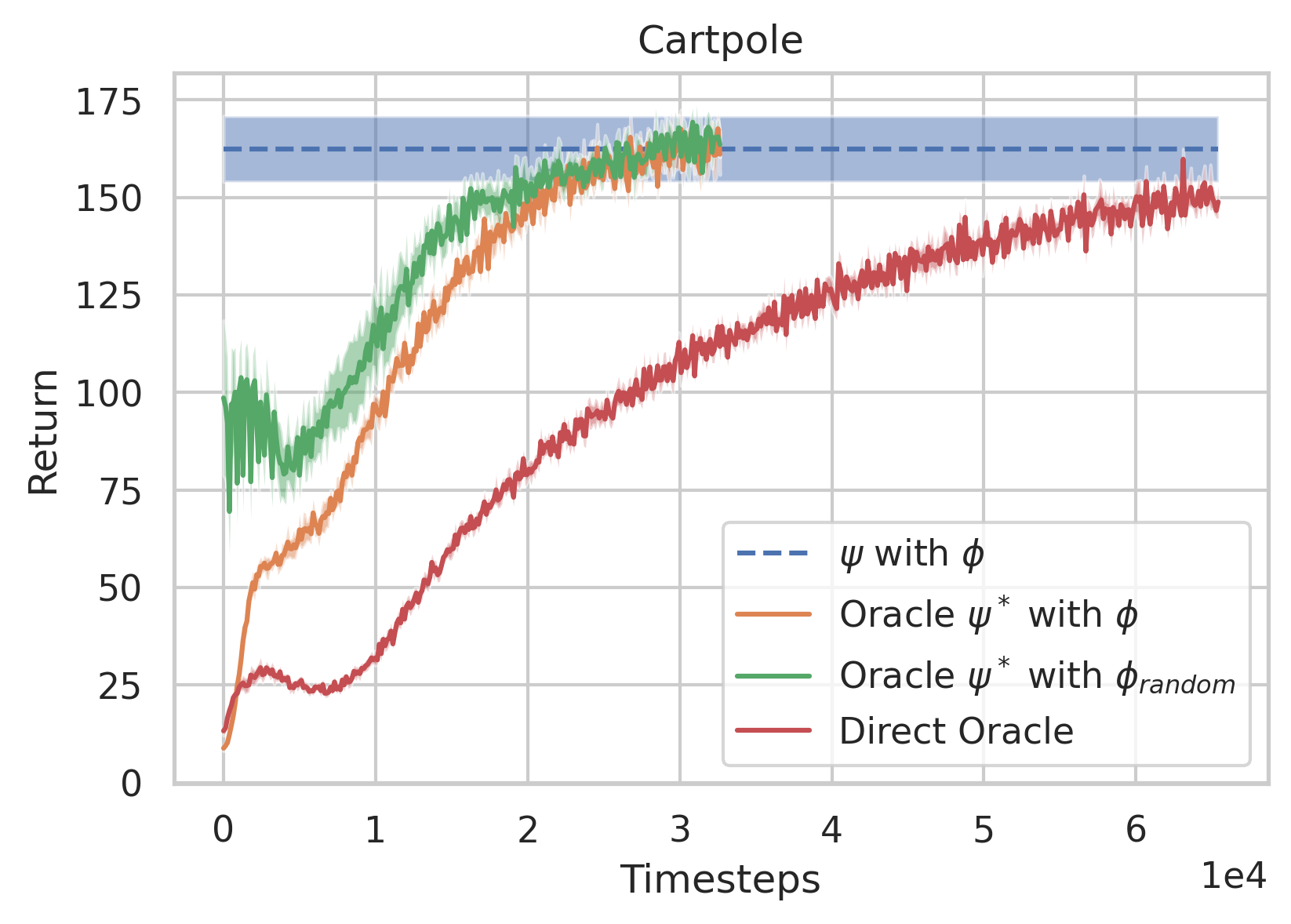}
	\caption{}
	\label{fig:cartpole_puppet}
 \end{subfigure}
  \begin{subfigure}[]{0.32\linewidth}
     \centering
	\includegraphics[width=\linewidth]{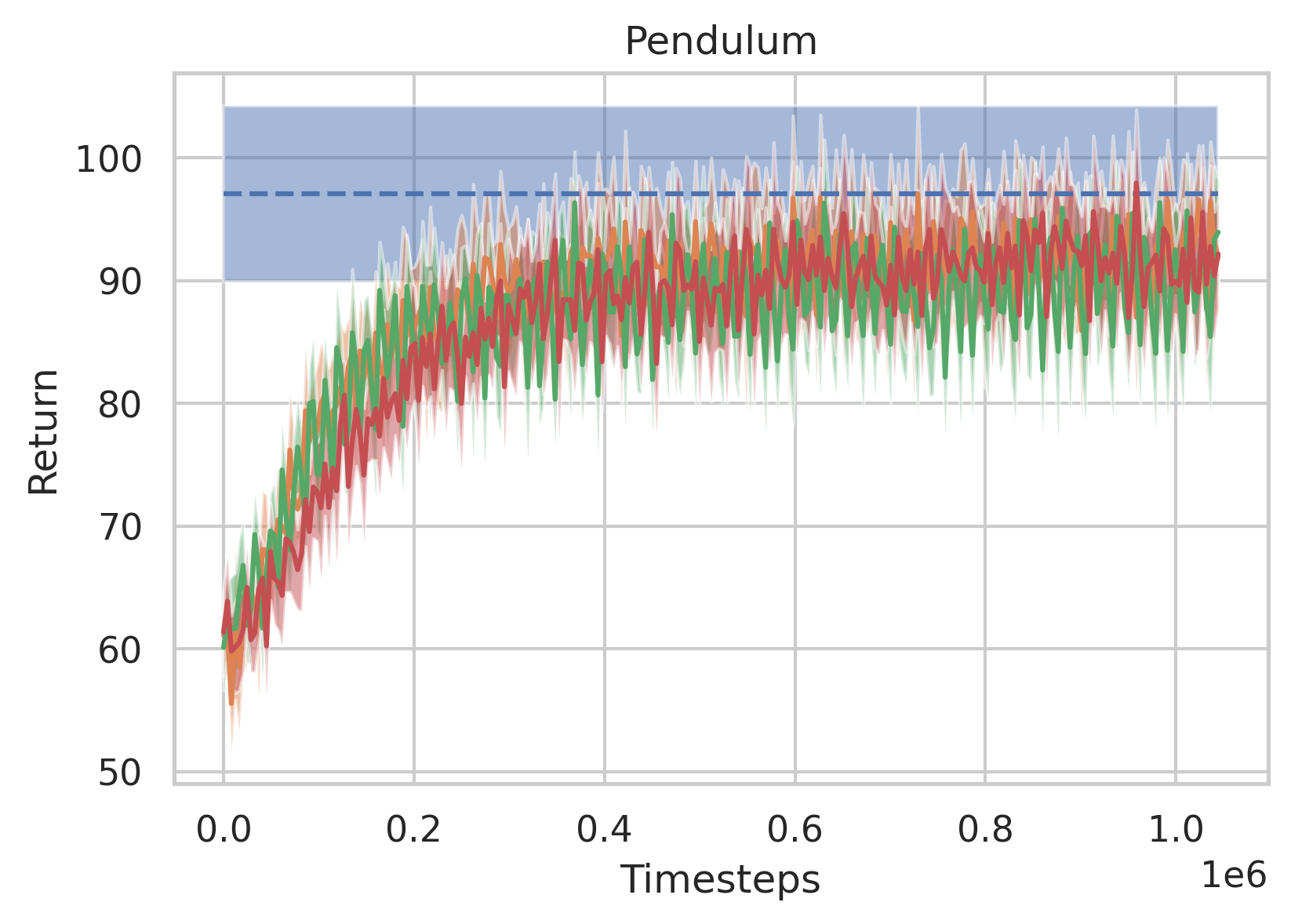}
	\caption{}
	\label{fig:pendulum_puppet}
 \end{subfigure}
  \begin{subfigure}[]{0.32\linewidth}
     \centering
	\includegraphics[width=\linewidth]{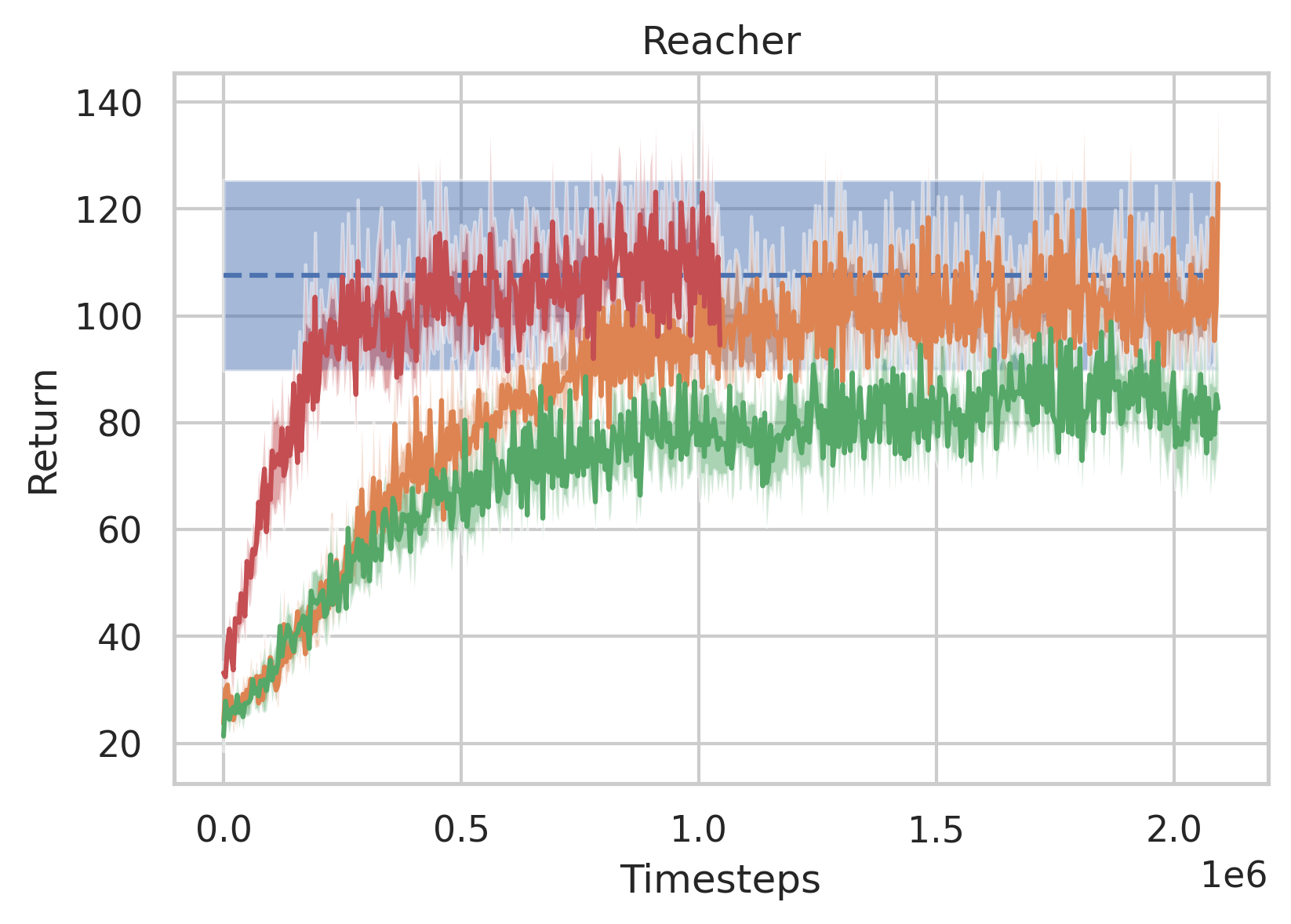}
	\caption{}
	\label{fig:reacher_puppet}
 \end{subfigure}
 \caption{Training curves of the different agents in (a) Goal-Conditioned Cartpole (b) Goal-Conditioned Pendulum (c) Goal-Conditioned Reacher. The ablations show that the train- and test-time \shaper\ learn near-optimal performance in comparison to the oracles.
 %Note that the Direct Oracle trains more slowly in Cartpole because the goal-conditioned and underlying environment reward functions are fairly aligned, whereas in Reacher they are conflicting. 
 Error bars denote the standard error of the mean across $10$ seeds of \learner\ trained against a single trained \shaper.
 }
 \label{fig:zero_shot_vs_oracles}
\end{figure*}

\subsection{Test-Time Manipulation}
% Recap, what exactly is our goal in Zero-Shot Test-Time manipulation
In test-time manipulation, the \shaper's objective is to maximise the score of the \textit{goal-conditioned objectives} described in Figure \ref{fig:goal_envs} at test-time. The \shaper\ needs to learn to introduce a backdoor during train-time and use the backdoor during test-time to fully control the \learner\ as explained in Section \ref{subsec:zero-shot}.
% We describe the environment-specific rewards and how these goals are parameterised in Figure \ref{fig:goal_envs}.
% we are running four experiments to test 1. joint performance of train- and test-time \shaper 2. individual performance of train-time \shaper 3. individual performance of test-time \shaper
To better understand the capability of our model, we investigate four different \shaper-\learner\ settings, which serve as ablations to study the individual and joint performance of the train- and test-time \shapers. 

\begin{enumerate}

% Vanilla ACT: both train- and test-time \shaper are trained with ES
\item \textbf{Test-Time \shaper\ $\psi$ with Train-Time \shaper\ $\phi$:} This is the algorithm described in \ref{alg:mop_test}. First, we train a \learner\ $\theta$ alongside a train-time \shaper\ $\phi$. We then evaluate the return of the test-time \shaper\ $\psi$ according to the goal-conditioned return. Both the train- and test-time \shapers\ are trained using ES. The test-time \shaper\ $\psi$ only gets a single shot against a \learner\ and is thus represented by a horizontal line in Figure \ref{fig:zero_shot_vs_oracles}.

% test-time Oracle \shaper (trained with PPO instead of zero-shot) with Vanilla ES train-time \shaper
\item \textbf{Test-Time Oracle $\psi^*$ with Train-Time \shaper\ $\phi$:} First, we optimise the \learner\ $\theta$ by training it alongside the above train-time \shaper\ $\phi$. Then, instead of ES, we use PPO to train the test-time \shaper\ $\psi^*$ against the \learner\ $\theta$. Unlike the test-time \shaper, the oracle $\psi^*$ is allowed to train against the pretrained and fixed \learner\ $\theta$ to maximise its returns, as described in Algorithm \ref{alg:mop_test_oracleppo} in Appendix \ref{app:pseudo}. 

% test-time Oracle \shaper (trained with PPO instead of zero-shot) with RANDOM train-time \shaper
\item \textbf{Test-Time Oracle $\psi^*$ with Random Train-Time \shaper\ $\phi_{\text{random}}$:} First, we obtain a \learner\ $\theta$ by training it alongside a \textit{random} train-time \shaper, $\phi_{\text{random}}$, with randomly initialised and fixed parameters. Next, we use PPO to train the test-time \shaper\ $\psi^*$ to maximise the goal-conditioned return.

% Baseline for joint performance
\item \textbf{Direct Oracle:} In this baseline, there is no cheap talk or \learner. We simply train a PPO agent to maximise the \textit{goal-conditioned return} $J$. It can observe the full state and directly output actions in the environment.

\end{enumerate}

\begin{figure*}[t]
 \centering
 \begin{subfigure}[]{0.4\linewidth}
     \centering
	\includegraphics[width=\linewidth]{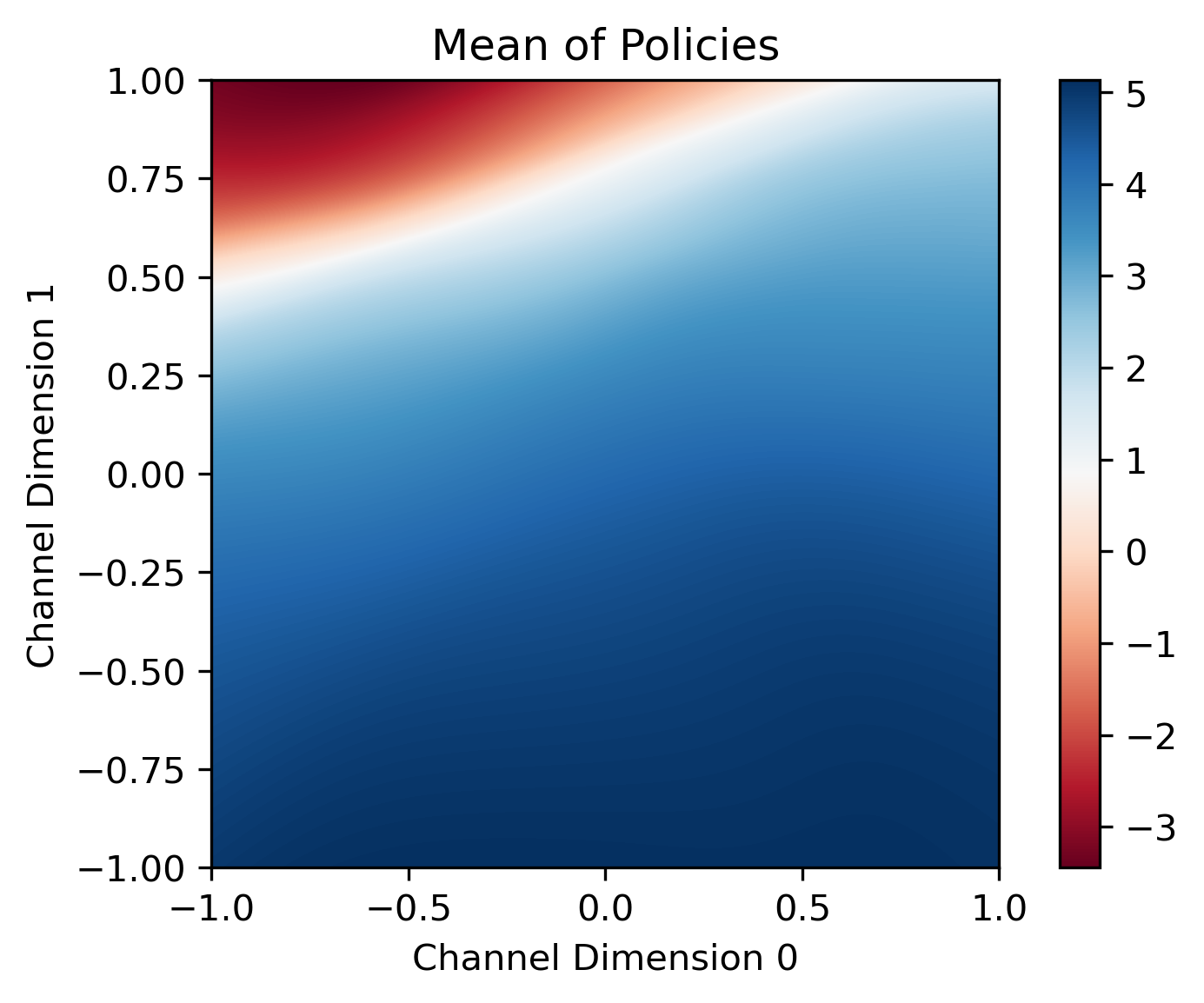}
	\caption{}
	\label{fig:mean_learned}
 \end{subfigure}
  \begin{subfigure}[]{0.4\linewidth}
     \centering
	\includegraphics[width=\linewidth]{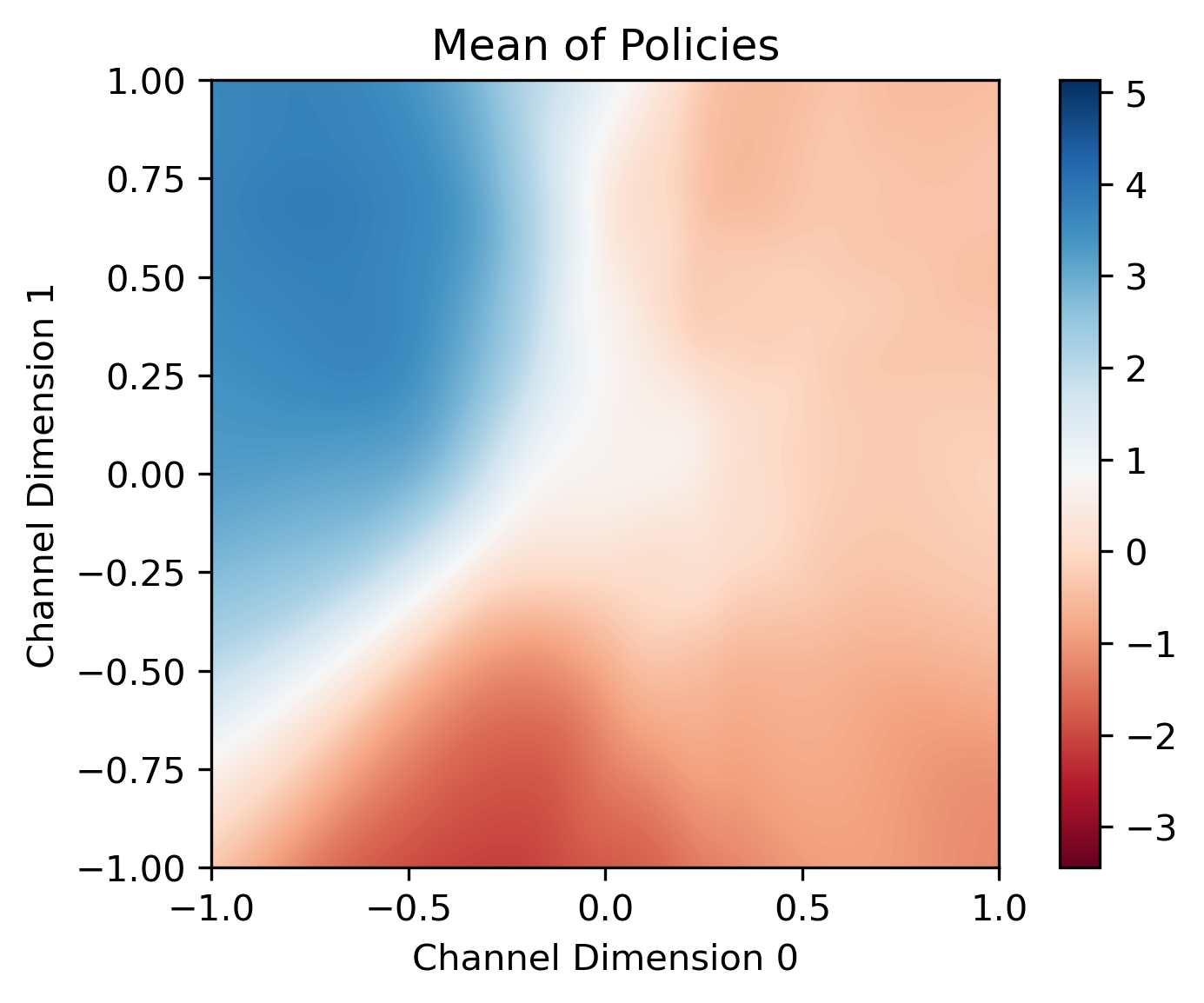}
	\caption{}
	\label{fig:mean_random}
 \end{subfigure}
  \begin{subfigure}[]{0.4\linewidth}
     \centering
	\includegraphics[width=\linewidth]{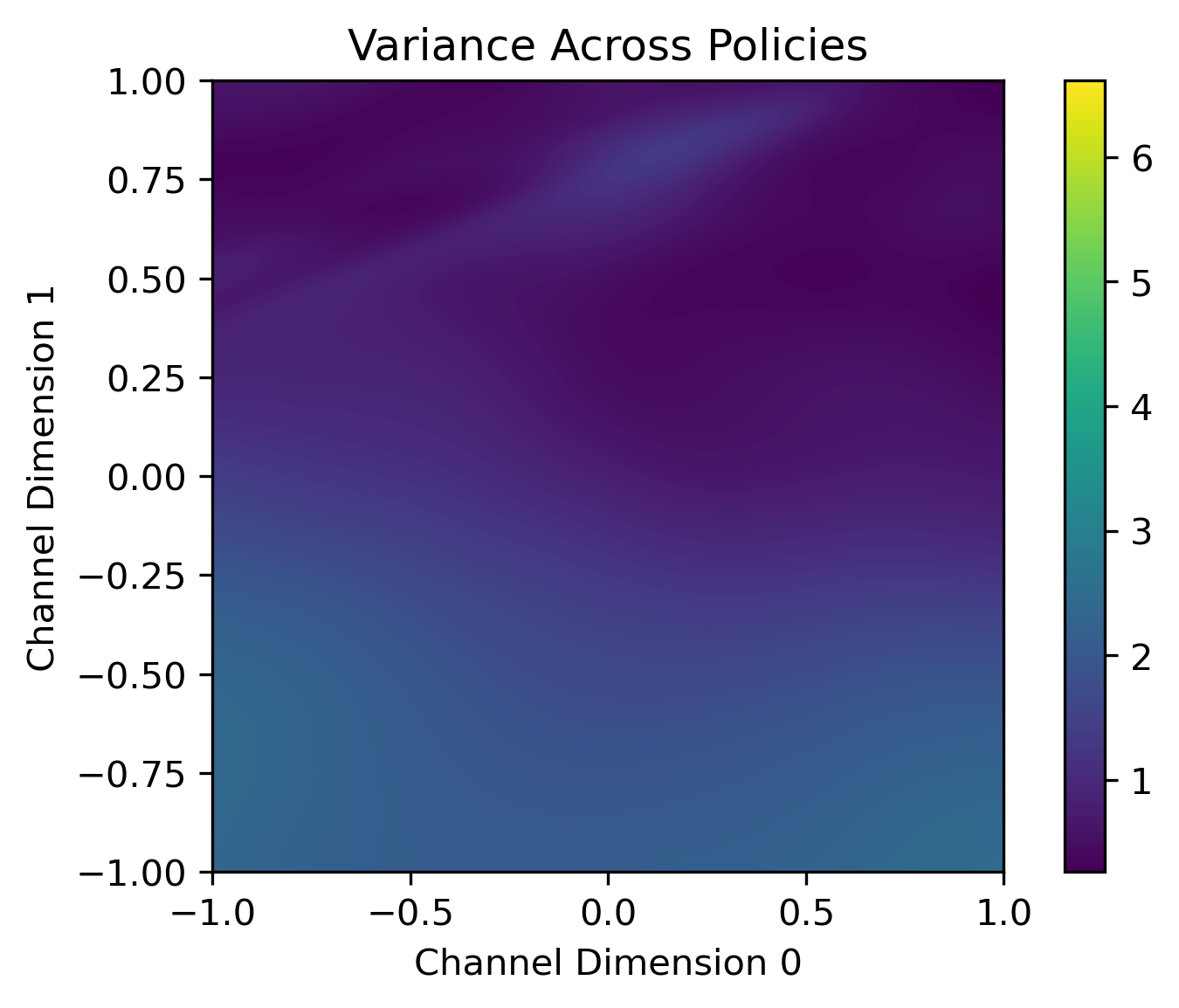}
	\caption{}
	\label{fig:var_learned}
 \end{subfigure}
   \begin{subfigure}[]{0.4\linewidth}
     \centering
	\includegraphics[width=\linewidth]{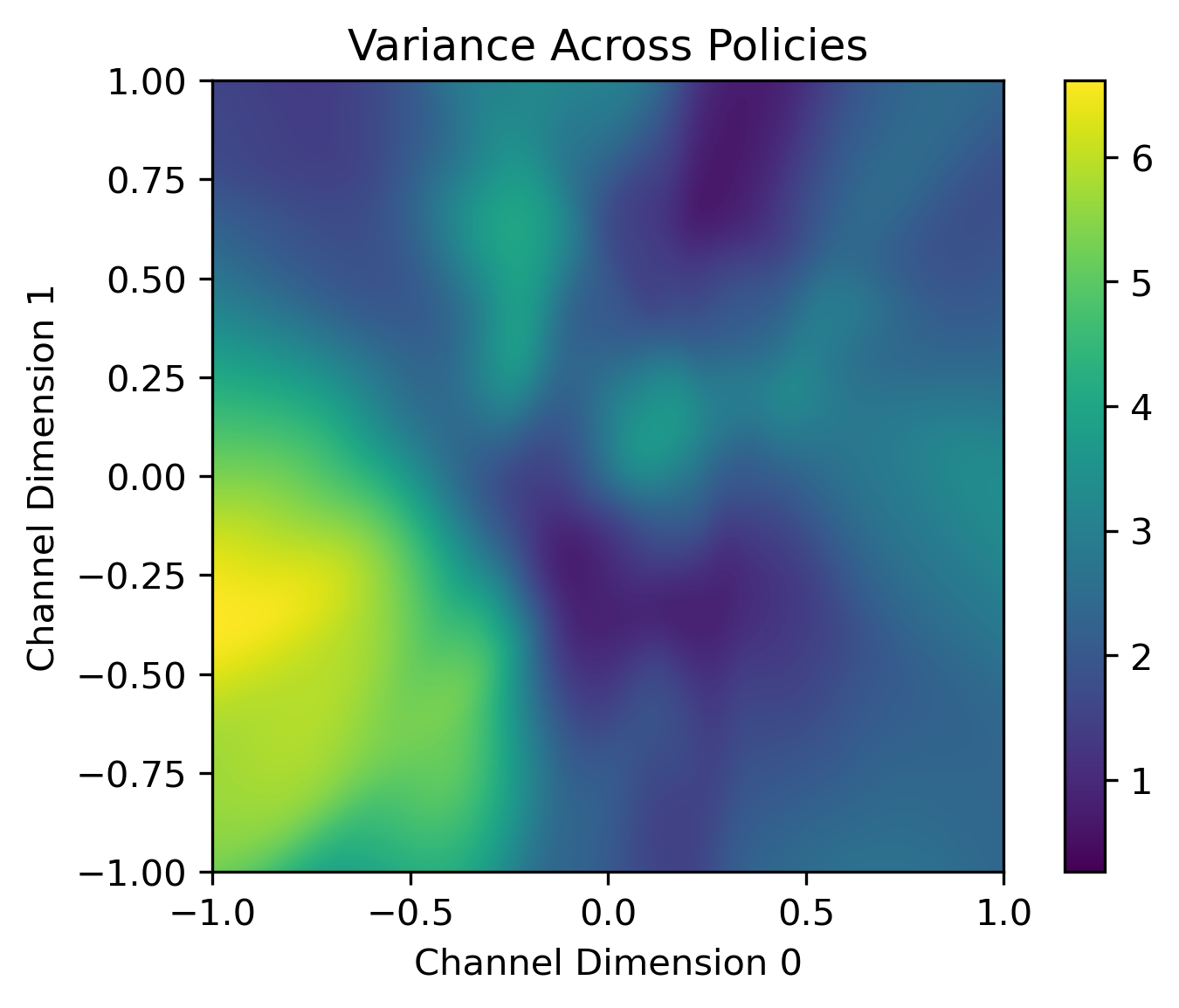}
	\caption{}
	\label{fig:var_random}
 \end{subfigure}
    \caption{We train 10 different \learner s\ alongside the Learned $\phi$ (a \& c), and 10 different \learner s\ alongside a randomly generated $\phi$ (b \& d) in the Pendulum environment. (a) and (b) show the mean of the policy output across the 10 \learner s\ as we vary the value of the \fname\, in a fixed randomly selected state. The policies trained with the learned $\phi$ achieve a much wider range of outputs. (c) and (d) show the variance of the policy output across the 10 \learner s. The policies trained with the learned $\phi$ display very little variance.
    %, implying that the learned $\phi$ shapes the \learner\ in a consistent way.
    }
    \label{fig:Pendulum_Vis}
\vspace{-15pt}
\end{figure*}
% First, analyse joint performance against oracle baseline
All results are shown in Figure \ref{fig:zero_shot_vs_oracles}. 
We can compare \textbf{(1)} and \textbf{(4)} to measure how effective the train-time \shaper\ $\phi$ and test-time \shaper\ $\psi$ are at achieving the maximal possible return jointly. 
As Figure \ref{fig:zero_shot_vs_oracles} shows, the train- and test-time \shapers\ perform near-optimally.
% We compare ES-optimised train-time \shaper with RANDOM train-time \shaper to see if the train-time \shaper actually does something

By comparing \textbf{(2)} and \textbf{(3)}, we can observe how effective $\phi$ is at \textit{shaping} $\theta$. In reacher, we can see that the test-time Oracle $\psi^*$ \textit{cannot} achieve the maximum performance with a random train-time \shaper. 

% As can be seen in Figure \ref{fig:zero_shot_vs_oracles}, the ES-optimised train-time \shaper\ performs near optimally. 
% Some explanation what the train-time \shaper actually learns
% We investigate this further in Figure \ref{fig:Pendulum_Vis} where we compare the range and variance of \learner s trained with ES-optimised \shapers\ $\phi$ and \learner s trained with random \shapers\ $\phi_{\text{random}}$ across different \fname\ values.
% Ablation: How well does the Zero-Shot test-time \shaper actually perform and what does he learn?
We compare \textbf{(1)} and \textbf{(2)} to see how effective the test-time \shaper\ $\psi$ is exploiting a given \learner\ $\theta$. $\psi$ achieves near-optimal performance even though it \textit{has never trained against the specific \learner\ $\theta$ or had access to its parameters}. In Figure \ref{fig:Pendulum_Vis} we show that this is possible because the train-time \shaper\ $\phi$ not only maximises the range of outputs that the cheap talk achieves, but it also does so in a \textit{consistent and low-variance} way. 

% \begin{figure*}[t!]
%     \centering
%     \includegraphics[width=0.9\textwidth]{Figures/Dummy_Plot.png}
%     \vspace{-10pt}
%     \caption{The meta-training curves of the test-time controller across different number of dimensions in the cheap talk channel}
%     \label{fig:Meta_Training_Ablations_Controller}
% \end{figure*}

\section{Conclusion \& Future Work}
In this paper, we propose a novel, minimum-viable, adversarial setting for RL agents, where the \shaper\ can only influence the \learner\ over \fname s, and can only do so with a deterministic function that only depends on the current state. By training an \shaper\ with adversarial cheap talk (ACT), we show that appending to the observations of a learning agent, even with strong constraints, is sufficient to drastically improve or \textit{decrease} a learning agent's train-time performance or introduce a backdoor to control the learning agent at test time completely.
%Our test-time ablation studies demonstrate that the train- and test-time \shapers\ achieve near-optimal performance individually and jointly, when compared against oracle baselines.
Furthermore, we provide an in-depth analysis of the behaviour of our \shapers. 
% At train-time, the \shaper\ learns to induce catastrophic interference to decrease training performance. In the test-time manipulation setting, the train-time \shaper\ learns to reduce the variance of the training process to reduce variance of potential \learner s the test-time \shaper\ could face. 

% Call to action to motivate our relevancy
As RL models become more widespread, we believe practitioners should consider this new class of minimum viable attacks. We propose using domain knowledge to filter out potentially controllable information as the first defence measure. Identifying these channels without domain knowledge is challenging: While there has been past work in identifying  task-irrelevant features in reinforcement learning \citep{lange2022lottery}, ACT features still contain task-relevant information since they are functions of the state. 
%Another approach would be to train independent agents on all subsets of the input dimensions and select the best-performing agent. While this would trivially defend against the train-time \shaper, it is often not computationally feasible. 
To defend against the test-time \shaper, one can detect when the input goes out-of-distribution \citep{lin2017detecting}. More work is needed to build more robust and practical defences.

% In future work, we will investigate different defence strategies, such as the identification of \fname s, and larger-scale input settings, such as Atari, where we can analyse the role of the \fname\ size.

% \section{Acknowledgments and Disclosure of Funding}

% Compute for this project was partially run on Oxford's Advanced Research Cluster (ARC). This work used the Cirrus UK National Tier-2 HPC Service at EPCC funded by the University of Edinburgh and EPSRC (EP/P020267/1). This work was also supported by an Oracle for Research Cloud Grant (19158657).
% Use unnumbered first level headings for the acknowledgments. All acknowledgments

% \section*{References}
\bibliography{main.bib}
\bibliographystyle{abbrvnat}

% References follow the acknowledgments. Use unnumbered first-level heading for
% the references. Any choice of citation style is acceptable as long as you are
% consistent. It is permissible to reduce the font size to \verb+small+ (9 point)
% when listing the references.
% Note that the Reference section does not count towards the page limit.
% \medskip

% {
% \small

%%%%%%%%%%%%%%%%%%%%%%%%%%%%%%%%%%%%%%%%%%%%%%%%%%%%%%%%%%%%
\newpage

\newpage

\onecolumn

\appendix

\section{Minimality of Cheap Talk MDPs}\label{app:proofs}

\subsection{Proof of Proposition \ref{prop:tabular}}\label{app:proof1}

\begin{proposition}{\ref{prop:tabular}}
In any Cheap Talk MDP, the policy of a \textbf{tabular} \learner\ is independent from its Adversary provided uniform initialisation along $\cM$, namely $\pi_0(\cdot \mid s_i, m_j) = \pi_0(\cdot \mid s_i, m_{j'}) \ \forall \ j, j'$.
\end{proposition}

\begin{proof} In a Cheap Talk MDP $\langle\cS, \cA, \cP, \cR, \gamma, \cM, f, \cJ \rangle$, a tabular \learner\ arbitrarily orders states as $\{s_1, \ldots, s_d \}$ and messages as $\{ m_1, \ldots, m_k\}$, where $d = |\cS|$ and $k = |\cM|$, and stores policies $\pi_t(\cdot \mid s_i, m_j)$ at time $t$ of the learning process for all $i \in [d], j \in [k]$. The argument follows identically for value functions. Assuming uniform initialisation along the $\cM$ axis means that 
\[ \pi_0(\cdot \mid s_i, m_j) = \pi_0(\cdot \mid s_i, m_{j'})\]
for all $j, j' \in [k]$. Now consider any two \shapers\ $f, g$ and their influence on two copies of the same \learner\ $V, W$ with respective policies $\pi, \lchi$. The only states encountered in the environment are of the form $(s, f(s))$ and $(s, g(s))$ respectively, so \learners\ only update the corresponding policies
\[ \pi_t(\cdot \mid s_i, f(s_i)) \qquad \text{and} \qquad \lchi_t(\cdot \mid s_i, g(s_i)) \,. \]
We prove by induction that these quantities are equal for all $t$. The base case holds by uniform initialisation along $\cM$; assume the claim holds for all fixed $0 \leq t \leq T$. The \learners\ update their policies at time $T+1$ according to the same learning rule, as a function of the transitions and returns under current and past policies $\pi_t$ and $\lchi_t$ respectively. Transitions take the form $(s, f(s), a, s', f(s'))$ for $V$ and $(s, g(s), a, s', g(s))$ for $W$, which have identical probabilities and returns because
\begin{align*}
\pi_t(a \mid s_i, f(s_i)) &= \lchi_t(a \mid s_i, g(s_i)) \, ; \\
\cP(s', f(s') \mid s, f(s), a) &= \cP(s', g(s') \mid s, g(s), a) \, ;\\
\cR(s, f(s), a) &= \cR(s, g(s), a)
\end{align*}
by inductive assumption and independence of $\cP, \cR$ from $\cM$. This implies that the Victims' policies $\pi_T(\cdot \mid s_i, f(s_i)) = \lchi_T(\cdot \mid s_i, g(s_i))$ are updated identically to
\[ \pi_{T+1}(\cdot \mid s_i, f(s_i)) = \lchi_{T+1}(\cdot \mid s_i, g(s_i)) \]
as required to complete induction. Note that this would not necessarily hold in non-tabular settings, where updating parameters $\theta$ of the function approximator for some state $(s_i, f(s_i))$ may alter the policy on some other state $(s_j, f(s_j))$. It now follows that trajectories $\tau = (s^k, f(s^k), a^k)_k$ for $V$ and $\omega = (s^k, g(s^k), a^k)_k$ for $W$ have identical probabilities and hence produce identical returns
\[ \mathbb{E}_{\tau \sim \pi_t} \left[ \cR(\tau) \right] = \mathbb{E}_{\omega \sim \lchi_t} \left[ \cR(\omega) \right] \]
at any timestep $t$ of the learning process, concluding independence from \shapers.
\end{proof}

\subsection{Proof of Proposition \ref{prop:optimal}}\label{app:proof2}

\begin{proposition}{\ref{prop:optimal}}
A \learner\ which is \textbf{guaranteed to converge to optimal policies in MDPs} will also converge to optimal policies in Cheap Talk MDPs, with an expected return equal to the optimal return for the corresponding no-channel MDP.
\end{proposition}

\begin{proof}
By assumption, the \learner\ is guaranteed to converge to an optimal policy $\bar{\pi}$ in any given Cheap Talk MDP $\langle \cS, \cA, \cP, \cR, \cM, f, \cJ, \gamma\rangle$, since a Cheap Talk MDP is itself an MDP with an augmented state space $\cS \times \cM$ and augmented transition/reward functions that are defined to be independent from $\cM$. Now $\bar{\pi}$ naturally induces a policy $\pi$ on the no-channel MDP, given by $\pi(\cdot \mid s) \coloneqq \bar{\pi}(\cdot \mid s, f(s))$, and in particular $Q(s, a) = \bar{Q}(s, f(s), a)$ by independence of transitions and rewards from $\cM$. Optimality of $\pi$ follows directly from the Bellman equation
\begin{align*}
Q(s, a) = \bar{Q}(s, f(s), a) &= \mathbb{E}_{s' \sim \cP(\cdot \mid s, a), r \sim \cR(\cdot \mid s, a)}\left[r + \gamma \max_{a' \in \cA} \bar{Q}(s', f(s'), a') \right] \\
&= \mathbb{E}_{s' \sim \cP(\cdot \mid s, a), r \sim \cR(\cdot \mid s, a)}\left[r + \gamma \max_{a' \in \cA} Q(s', a') \right] \,. 
\end{align*}

Now trajectories $\bar{\tau} = (s^k, f(s^k), a^k)_k$ and $\tau = (s^k, a^k)_k$ have identical probability and return under $\pi$ and $\bar{\pi}$ respectively, so the \learner\ has expected return
\[ \mathbb{E}_{\bar{\tau} \sim \bar{\pi}} \left[ \cR(\bar{\tau}) \right] = \mathbb{E}_{\tau \sim \pi} \left[ \cR(\tau) \right] \]
which is the optimal expected return of the original no-channel MDP.
\end{proof}

\subsection{Further Informal Discussion}\label{app:informal}

Consider a \psname\ $\langle\cS, \cA, \cP, \cR, \gamma, \cM, f, \cJ \rangle$. For a fixed training / testing run of the \learner\ on the MDP, the \shaper\ outputs a \fname\ $f(s)$ at each step according to a fixed deterministic function $f : \cS \to \cM$. Below we elaborate informally on the claims that Adversaries cannot (1) occlude the ground truth, (2) influence the environment dynamics / reward functions, (3) see the \learner's actions or parameters, (4) inject stochasticity, or (5) introduce non-stationarity.

\begin{enumerate}[label={(\arabic*)}, leftmargin=20pt]
\item The \fname\ is \textit{appended} to the state $s$ and the \learner\ acts with full visibility of the ground truth (state) $s$ according to its policy: $a \sim \pi(\cdot \mid s, f(s))$.

\item The transition and reward functions $\cP, \cR$ are defined to be independent from $\cM$. Formally we have $\cP(\cdot \mid s, m, a) = \cP(\cdot \mid s, m', a)$ for all $m, m' \in \cM$ (similarly for $\cR$), so the \shaper's choice of \fname\ $m = f(s)$ cannot influence $\cP$ or $\cR$.

\item $f : \cS \to \cM$ is defined as a function of $\cS$ only, so the \shaper\ cannot condition its policy based on the \learner's actions or parameters (i.e. it cannot see or influence them).

\item $f$ is a deterministic function, so $\pi(\cdot \mid s, f(s))$ is a distribution only on actions $\cA$. The transition and reward functions are independent from $f$, so they are distributions only on state-action pairs $\cS \times \cA$. It follows that the \shaper\ injects no further stochasticity into the MDP.

\item $f$ is static for a fixed training / testing run, so $s_t = s_{t'}$ implies $f(s_t) = f(s_{t'})$ for all timesteps $t, t'$ in the run. It follows that any given \learner\ policy $\pi$ is stationary, namely $\pi(\cdot \mid s_t, f(s_t)) = \pi(\cdot \mid s_{t'}, f(s_{t'}))$ for all $s_t = s_{t'}$. Since $\cP$ and $\cR$ are stationary (as defined by a standard MDP) and independent from $\cM$, their stationarity is also preserved. \qedhere
\end{enumerate}

Finally, we discuss the possibility of further weakening components of a \psname, and conclude that all such variants (A-E) bring no advantage or reduce to regular MDPs.

\begin{enumerate}[label={(\Alph*)}, leftmargin=20pt]
\item Removing the channel $\cM$ or the policy $f : \cS \to \cM$ would result in the \learner\ being completely independent from the \shaper, so no adversarial influence could be exerted whatsoever.
\item Restricting the capacity of $\cM$ to a certain number of bits would further restrict an \shaper's range of influence, so one could say that the \textit{truly} minimum-viable setting is to impose a set of size $|\cM| = 1$. However, cheap talk is still cheap talk when varying capacity, and there is no reason to arbitrarily restrict the size to $1$ if we are to apply our setting to complex environments likely requiring more than a single bit of communication to witness interesting results.
\item Not allowing \shapers\ to see states, namely removing $\cS$ as inputs to $f$, yields a function $f : \{ 0 \} \to \cM$ which always outputs the same \fname $f(0) = m \in \cM$. This is equivalent to the previous restriction of imposing a set $\cM$ of size 1, since in this case any function $f : \cS \to \cM$ would have to output the unique element $f(s) = m$ for all input states $s$.
\item The \shaper\ must have some objective function $\cJ$ in order for an adversarial setting to make sense -- removing it would remove the \shaper's reason to exist, since it would have no incentive to learn parameters that influence the \learner\ according to some goal.
\item Restricting the function class of objectives $\cJ$ is a valid minimisation of the setting, but simply restricts our interest in the setting itself. The setting should at the very least allow for adversarial objectives of the form $\cJ = \pm J$, as we consider in the train-time setting. In test-time, our aim is to show how \shapers\ can exert arbitrary control over \learners\ despite cheap talk restrictions, and we therefore consider more general objective functions.
\end{enumerate}

\pagebreak

\section{Pseudocode}
\label{app:pseudo}

% \vspace{-400pt}
\begin{algorithm}[htp!]
\caption{Test-time ACT}
\label{alg:mop_test}
\begin{algorithmic}[1]
\STATE Initialize train-time ACT parameters $\phi$
\STATE Initialize test-time ACT parameters $\psi$
\FOR{$m=0$ {\bfseries to} $M$} %\COMMENT{M is number of meta-episodes}
    \STATE Sample $\phi_n \sim \phi + \sigma\epsilon_{n}$ where $\epsilon_{1}, ..., \epsilon_{N} \sim \mathcal{N}(0, I)$
    \STATE Sample $\psi_n \sim \psi + \sigma\epsilon_{n}$ where $\epsilon_{1}, ..., \epsilon_{N} \sim \mathcal{N}(0, I)$
    \FOR{$n=0$ {\bfseries to} $N$}
    \STATE Initialize policy params $\theta$
    \STATE rewards = []
    \FOR{$e=0$ {\bfseries to} $E$} %\COMMENT{E is number of training episodes}
        \STATE s = env.reset()
        \WHILE{not done} %\COMMENT{T is episode length}
            \STATE $m = f_{\phi_{n}}(s)$
            \STATE $\bar{s}$ = [$s$, $m$]
            \STATE $a \sim \pi_\theta(\cdot\mid \bar{s})$
            \STATE $r, s$ = env.step($a$)
        \ENDWHILE
        \STATE Update $\theta$ using PPO to maximise its return $J$
    \ENDFOR
    \FOR{$i=0$ {\bfseries to} $I$} %\COMMENT{E is number of training episodes}
        \STATE $g = \text{env.getgoal()}$
        \STATE $s = \text{env.reset()}$
        \WHILE{not done} %\COMMENT{T is episode length}
            \STATE $m = f_{\psi_{n}}(s, g)$
            \STATE $\bar{s}$ = [$s$, $m$]
            \STATE $a \sim \pi_\theta(\cdot\mid \bar{s})$
            \STATE $r, s$, done = env.step($a$)
            \STATE $r^{G}_t = R^{G}(s, a, g)$
            \STATE scores.append($r^{G}_t$)
        \ENDWHILE
    \ENDFOR
    \STATE $\cJ$ = $\text{sum(scores)} / I$
    \ENDFOR
    \STATE Update $\phi$ and $\psi$ using ES to maximise goal-conditioned objective $\cJ$
\ENDFOR
\end{algorithmic}
\end{algorithm}

\pagebreak

\begin{algorithm}[htp!]
\caption{Test-time Oracle PPO ACT}
\label{alg:mop_test_oracleppo}
\begin{algorithmic}[1]
\STATE Initialize train-time ACT parameters $\phi$
\STATE Obtain trained $\phi$, $\theta$ from Algorithm \ref{alg:mop_test}
\STATE Initialize test-time ACT parameters $\psi^*$
\FOR{$i=0$ {\bfseries to} $I$} %\COMMENT{E is number of training episodes}
    \STATE s = env.reset()
    \WHILE{not done} %\COMMENT{T is episode length}
        \STATE $m \sim \pi_{\psi^{*}}(\cdot\mid s)$
        \STATE $\bar{s}$ = [$s$, $m$]
        \STATE $a \sim \pi_\theta(\cdot\mid \bar{s})$
        \STATE $r, s$, done = env.step($a$)
        \STATE $r^{S}_t = R^{S}(s, a)$
        \STATE rewards.append($r^{S}_t$)
    \ENDWHILE
    \STATE Update $\psi^*$ using PPO to maximise $\cJ$
\ENDFOR
\end{algorithmic}
\end{algorithm}

\pagebreak

\begin{algorithm}[htp!]
\caption{Test-time Random Shaper}
\label{alg:mop_test_random}
\begin{algorithmic}[1]
\STATE Initialize train-time ACT parameters $\phi_{\text{random}}$
\STATE Initialize policy params $\theta$
\STATE rewards = []
\FOR{$e=0$ {\bfseries to} $E$} %\COMMENT{E is number of training episodes}
    \STATE s = env.reset()
    \WHILE{not done} %\COMMENT{T is episode length}
        \STATE $m = f_{\phi_{\text{random}}}(s)$
        \STATE $\bar{s}$ = [$s$, $m$]
        \STATE $a \sim \pi_\theta(\cdot\mid \bar{s})$
        \STATE $r, s$ = env.step($a$)
    \ENDWHILE
    \STATE Update $\theta$ using PPO to maximise $J$
\ENDFOR
\STATE Initialize test-time ACT parameters $\psi^*$
\FOR{$i=0$ {\bfseries to} $I$} %\COMMENT{E is number of training episodes}
    \STATE s = env.reset()
    \WHILE{not done} %\COMMENT{T is episode length}
        \STATE $m \sim \pi_{\psi^*}(\cdot\mid s)$
        \STATE $\bar{s}$ = [$s$, $m$]
        \STATE $a \sim \pi_\theta(\cdot\mid \bar{s})$
        \STATE $r, s$ = env.step($a$)
        \STATE $r^{S}_t = R^{S}(s, a)$
        \STATE rewards.append($r^{S}_t$)
    \ENDWHILE
    \STATE Update $\psi^*$ using PPO to maximise $\cJ$
\ENDFOR
    % \STATE $\cJ$ = $\frac{\text{sum(rewards)}}{E}$

\end{algorithmic}
\end{algorithm}

\pagebreak

\section{Ablations}
\label{app:ablations}

\begin{figure*}[hb]
 \begin{subfigure}[]{0.49\linewidth}
     \centering
	\includegraphics[width=\linewidth]{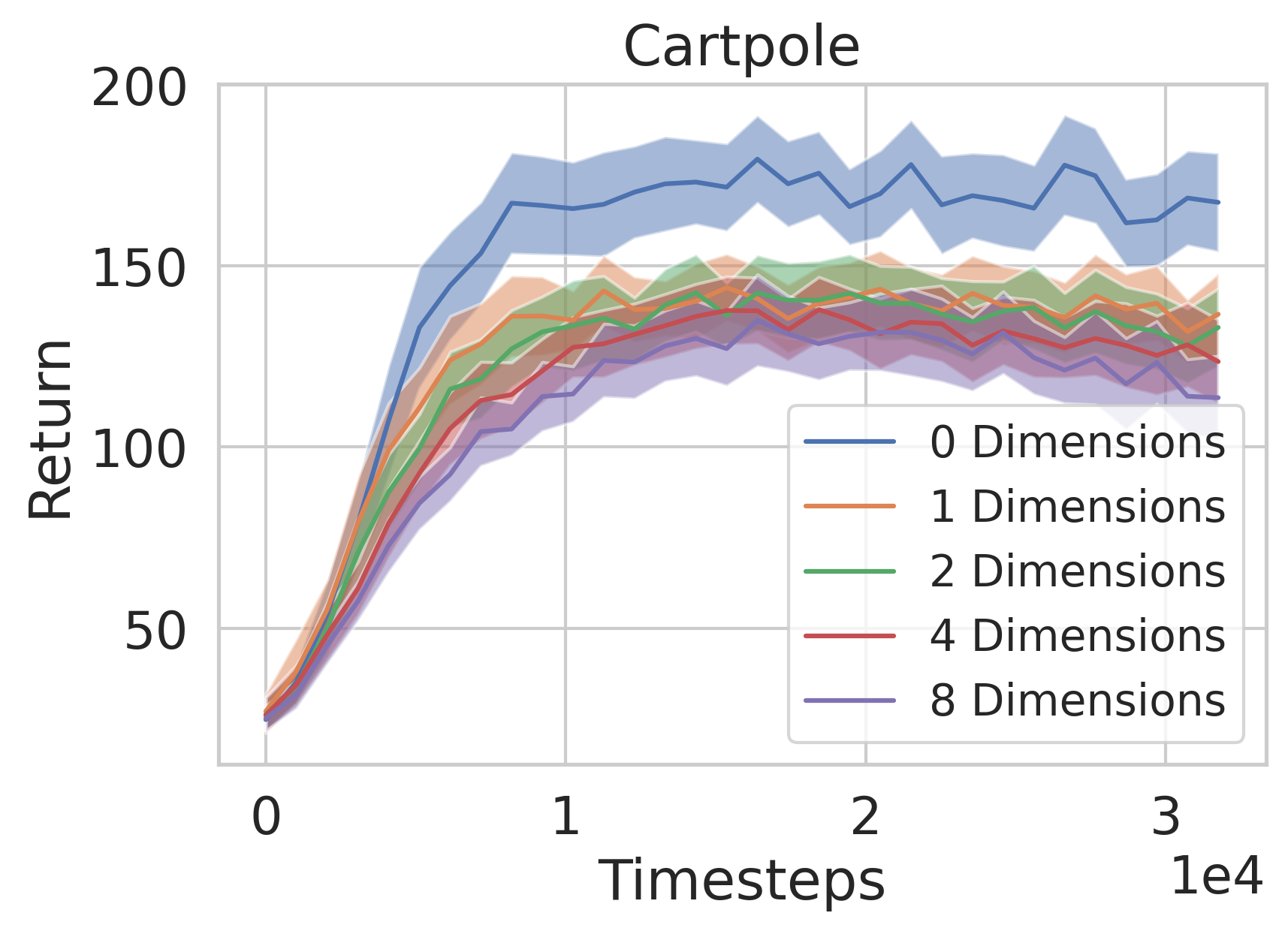}
	\caption{}
	\label{fig:ablation_dimensions_curves}
 \end{subfigure}
  \begin{subfigure}[]{0.49\linewidth}
     \centering
	\includegraphics[width=\linewidth]{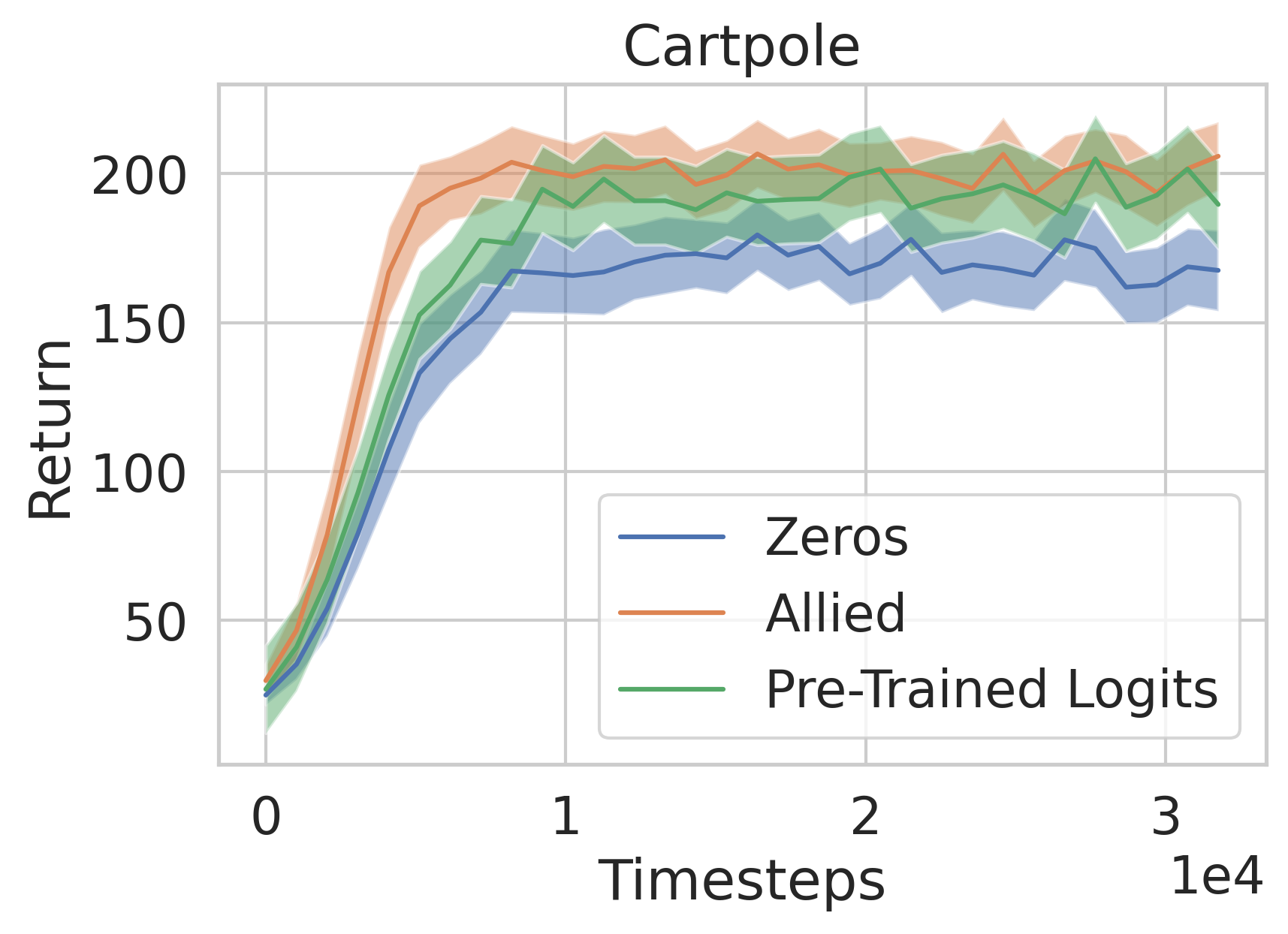}
	\caption{}
	\label{fig:pro_oracle_curves}
 \end{subfigure}
    \caption{(a) Ablations on the different number of cheap talk dimensions for the \shaper\ in Cartpole. We find that for a low-dimensional environment like Cartpole, the \shaper\ does not achieve much marginal improvement from increasing the number of channels, suggesting that there may be some limit to the amount that it can harm performance. (b) Comparing the ally with an \shaper\ that outputs pre-trained logits in Cartpole. We find that the allied ACT still performs better, implying that it is outputting features that are more useful than logits from a pre-trained policy. Error bars denote the standard error across $10$ seeds of a \learner\ trained against a single meta-trained \shaper.}
    \label{fig:ablation_curves}
\end{figure*}

\pagebreak

\pagebreak

\section{Pendulum Ablation}
\label{app:pendulum-ablation}

\begin{figure*}[h]
 \centering
\includegraphics[width=\linewidth]{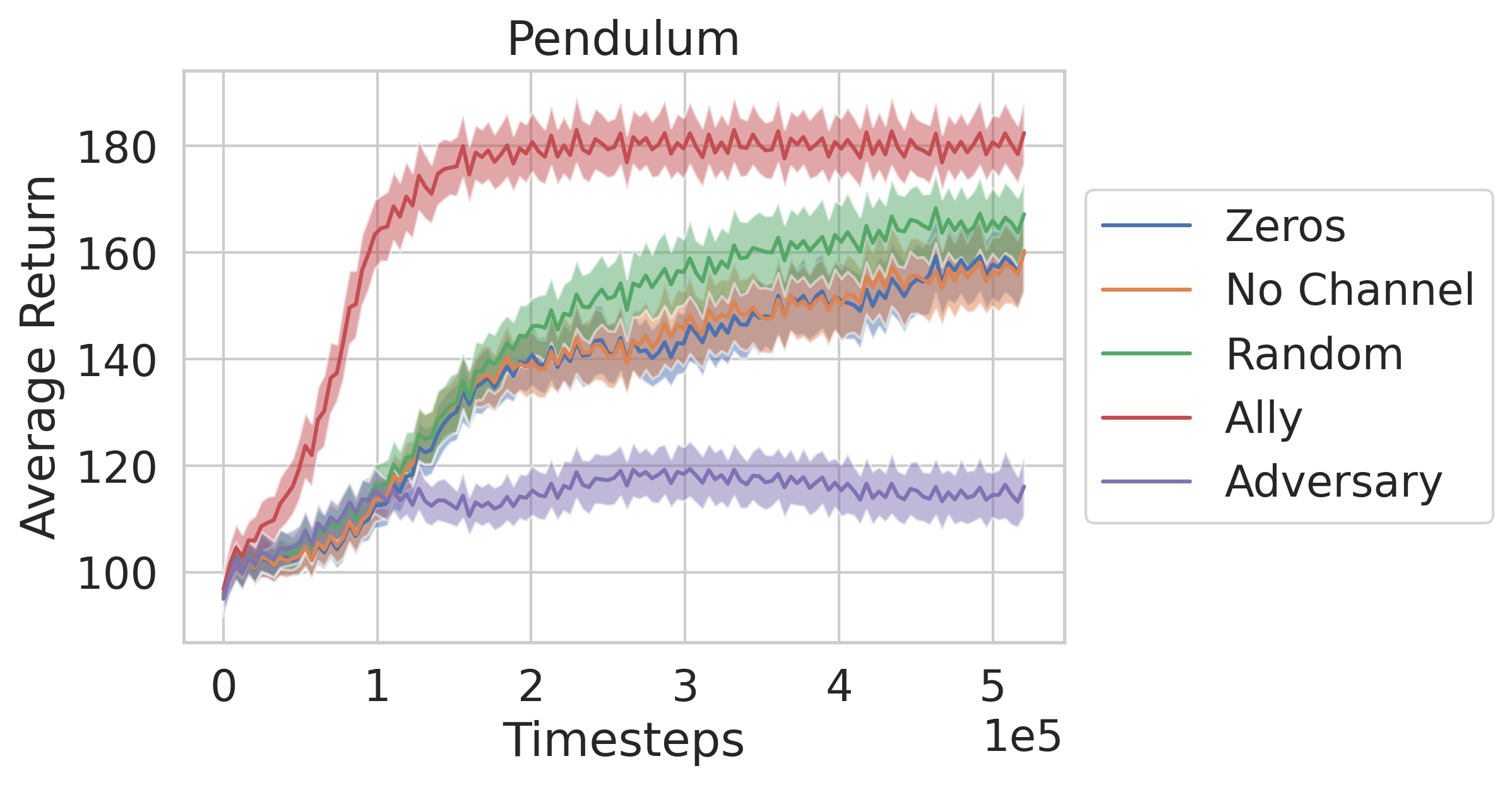}
\caption{Interestingly, it seems like random network features improved performance in Pendulum. To make sure this was not due to network initialisation effects, we ran an ablation where we removed the cheap talk channel. It achieves about the same performance as a channel with zeros, which implies that the performance difference is not due to network initialisation.}
 \label{fig:ablation_pendulum}
\end{figure*}

\pagebreak

\section{Hyperparameter Details}
\label{app:hyperparameter-details}
We report the hyperparameter values used for each environment in our experiments. Our PPO implementation uses observation normalisation (a common design choice for PPO), which means that the attack is invariant to the \textit{range} of values outputted over the cheap talk channel. 

\begin{table}[h]
\centering
\caption{Important parameters for the Cartpole environment}
\label{tab:cartpole_hyperparams}
\begin{tabular}{l|l}
Parameter & Value \\ \hline
State Size & 4 \\
\fname\ Size & 2 \\
\fname\ Range & -2$\pi$, 2$\pi$ \\
Number of Environments & 4 \\
Maximum Grad Norm & 0.5 \\
Number of Updates & 32 \\
Update Period & 256 \\
Outer Discount Factor $\gamma$ & 0.99 \\
Number of Epochs per Update & 16 \\
PPO Clipping $\epsilon$ & 0.2 \\
General Advantage Estimation $\lambda$ & 0.95 \\
Critic Coefficient & 0.5 \\
Entropy Coefficient & 0.01 \\
Learning Rate & 0.005 \\
Population Size & 1024 \\
Number of Generations & 2049 \\
Outer Agent (OA) Hidden Layers & 2 \\
OA Size of Hidden Layers & 64 \\
OA Hidden Activation Function & ReLU \\
OA Output Activation Function & Tanh \\
Inner Agent (IA) Actor Hidden Layers & 2 \\
IA Size of Actor Hidden Layers & 32 \\
IA Number of Critic Hidden Layers & 2 \\
IA Size of Critic Hidden Layers & 32 \\
IA Activation Function & Tanh \\
Number of Rollouts & 4 \\
\end{tabular}

\end{table}

\pagebreak

\begin{table}[h]
\centering
\caption{Important parameters for the Pendulum environment}
\label{tab:pendulum_hyperparams}
\begin{tabular}{l|l}
Parameter & Value \\ \hline
State Size & 3 \\
\fname\ Size & 2 \\
\fname\ Range & -2$\pi$, 2$\pi$ \\
Number of Environments & 16 \\
Maximum Grad Norm & 0.5 \\
Number of Updates & 128 \\
Update Period & 256 \\
Outer Discount Factor $\gamma$ & 0.95 \\
Number of Epochs per Update & 16 \\
PPO Clipping $\epsilon$ & 0.2 \\
General Advantage Estimation $\lambda$ & 0.95 \\
Critic Coefficient & 0.5 \\
Entropy Coefficient & 0.005 \\
Learning Rate & 0.02 \\
Population Size & 768 \\
Number of Generations & 2049 \\
Outer Agent (OA) Hidden Layers & 2 \\
OA Size of Hidden Layers & 64 \\
OA Hidden Activation Function & ReLU \\
OA Output Activation Function & Tanh \\
Inner Agent (IA) Actor Hidden Layers & 1 \\
IA Size of Actor Hidden Layers & 32 \\
IA Number of Critic Hidden Layers & 1 \\
IA Size of Critic Hidden Layers & 32 \\
IA Activation Function & Tanh \\
Number of Rollouts & 4 \\
\end{tabular}

\end{table}

\pagebreak

\begin{table}[h]
\centering
\caption{Important parameters for the Reacher environment}
\label{tab:reacher_hyperparams}
\begin{tabular}{l|l}
Parameter & Value \\ \hline
State Size & 10 \\
\fname\ Size & 4 \\
\fname\ Range & -2$\pi$, 2$\pi$ \\
Number of Environments & 32 \\
Maximum Grad Norm & 0.5 \\
Number of Updates & 256 \\
Update Period & 128 \\
Outer Discount Factor $\gamma$ & 0.99 \\
Number of Epochs per Update & 10 \\
PPO Clipping $\epsilon$ & 0.2 \\
General Advantage Estimation $\lambda$ & 0.95 \\
Critic Coefficient & 0.5 \\
Entropy Coefficient & 0.0005 \\
Learning Rate & 0.004 \\
Population Size & 128 \\
Number of Generations & 2049 \\
Outer Agent (OA) Hidden Layers & 2 \\
OA Size of Hidden Layers & 64 \\
OA Hidden Activation Function & ReLU \\
OA Output Activation Function & Tanh \\
Inner Agent (IA) Actor Hidden Layers & 2 \\
IA Size of Actor Hidden Layers & 128 \\
IA Number of Critic Hidden Layers & 2 \\
IA Size of Critic Hidden Layers & 128 \\
IA Activation Function & ReLU \\
Number of Rollouts & 4 \\
\end{tabular}

\end{table}

\pagebreak

\begin{table}[h]
\centering
\caption{Important parameters for the Minatar environments}
\label{tab:minatar_hyperparams}
\begin{tabular}{l|l}
Parameter & Value \\ \hline
State Size & 400 \\
\fname\ Size & 32 \\
\fname\ Range & -2$\pi$, 2$\pi$ \\
Number of Environments & 64 \\
Maximum Grad Norm & 0.5 \\
Number of Updates & 1024 \\
Update Period & 256 \\
Outer Discount Factor $\gamma$ & 0.99 \\
Number of Epochs per Update & 32 \\
PPO Clipping $\epsilon$ & 0.2 \\
General Advantage Estimation $\lambda$ & 0.95 \\
Critic Coefficient & 0.5 \\
Entropy Coefficient & 0.01 \\
Learning Rate & 3e-4 \\
Population Size & 128 \\
Number of Generations & 256 \\
Outer Agent (OA) Hidden Layers & 2 \\
OA Size of Hidden Layers & 64 \\
OA Hidden Activation Function & ReLU \\
OA Output Activation Function & Tanh \\
Inner Agent (IA) Actor Hidden Layers & 2 \\
IA Size of Actor Hidden Layers & 256 \\
IA Number of Critic Hidden Layers & 2 \\
IA Size of Critic Hidden Layers & 256 \\
IA Activation Function & ReLU \\
Number of Rollouts & 1 \\
\end{tabular}

\end{table}

%%PERHAPS CONDENSE THE NEXT TWO PARAGRAPHS AND WRITE ABOUT HOW PAST ADVERSARIAL ATTACKS IMPLICITLY INFLUENCE THE UNDERLYING ENVIRONMENT DYNAMICS.
% Previously, adversarial attacks have famously been shown to work at test-time against supervised neural networks [CITE]. Follow-up work has since demonstrated that similar attacks can work at test-time against reinforcement-learning agents [CITE]. However, the attack settings considered in these works are often unrealistic. Actors usually cannot add arbitrary perturbation to another agent's inputs. In reinforcement learning settings, such perturbations can arguably be seen as changing the environment's underlying dynamics by obscuring relevant information. Furthermore, these attacks usually require access to the trained agent's weights and parameters in order to generate the adversarial inputs.

% Train-time adversarial attacks in reinforcement learning has been shown to ADFLKJBGLDJVL [CITE]. However, these works also consider fairly unrealistic settings in which the attacker has access to ADFBGDFVFBG.

\pagebreak

\section{RARL Hyperparameter Details}
\label{app:hyperparameter-rarl}

\begin{table}[h]
\centering
\caption{RARL Cartpole Parameters}
\label{tab:rarl_cartpole_hyperparams}
\begin{tabular}{l|l}
Parameter & Value \\ \hline
State Size & 4 \\
\fname\ Size & 2 \\
\fname\ Range & -2$\pi$, 2$\pi$ \\
Maximum Grad Norm & 0.5 \\
Total Number of Adversary and Learner Updates & 100 \\
Number of Learner Update Steps per Adversary Update & 8 \\
Number of Adversary Update Steps per Learner Update & 8 \\
Update Period & 256 \\
Outer Discount Factor $\gamma$ & 0.99 \\
Number of Epochs per Update & 16 \\
PPO Clipping $\epsilon$ & 0.2 \\
General Advantage Estimation $\lambda$ & 0.95 \\
Critic Coefficient & 0.5 \\
Entropy Coefficient & 0.01 \\
Learning Rate & 0.005 \\
Population Size & 1024 \\
Number of Generations & 2049 \\
Outer Agent (OA) Hidden Layers & 2 \\
OA Size of Hidden Layers & 64 \\
OA Hidden Activation Function & ReLU \\
OA Output Activation Function & Tanh \\
Inner Agent (IA) Actor Hidden Layers & 2 \\
IA Size of Actor Hidden Layers & 32 \\
IA Number of Critic Hidden Layers & 2 \\
IA Size of Critic Hidden Layers & 32 \\
IA Activation Function & Tanh \\
Number of Rollouts & 16 \\
\end{tabular}

\end{table}

\pagebreak

\section{Extra Visualisations}
\label{app:visualisations}

\begin{figure*}[hb!]
 \centering
 \begin{subfigure}[]{0.35\linewidth}
     \centering
	\includegraphics[width=\linewidth]{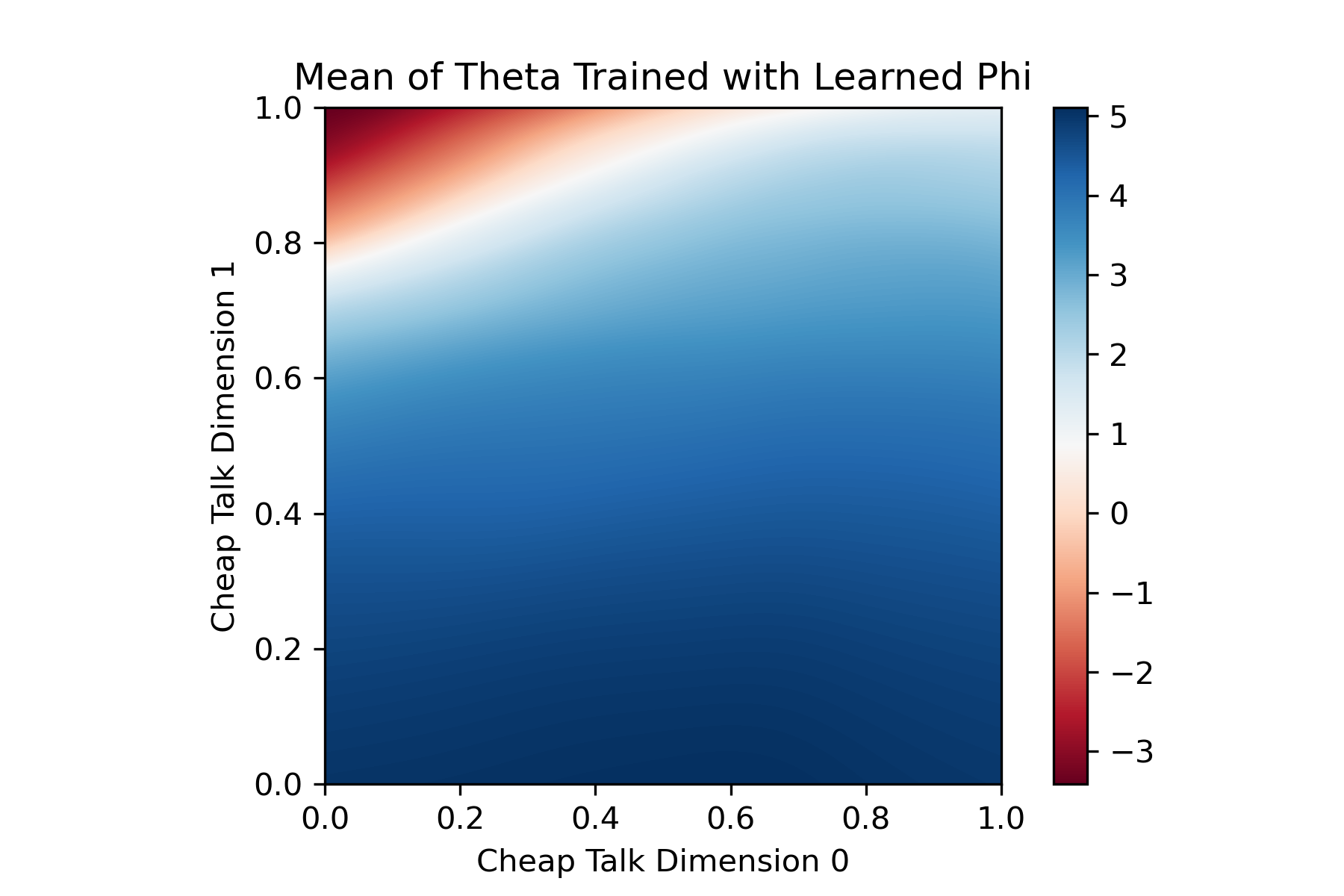}
	\caption{}
	\label{fig:mean_learned_0}
 \end{subfigure}
  \begin{subfigure}[]{0.35\linewidth}
     \centering
	\includegraphics[width=\linewidth]{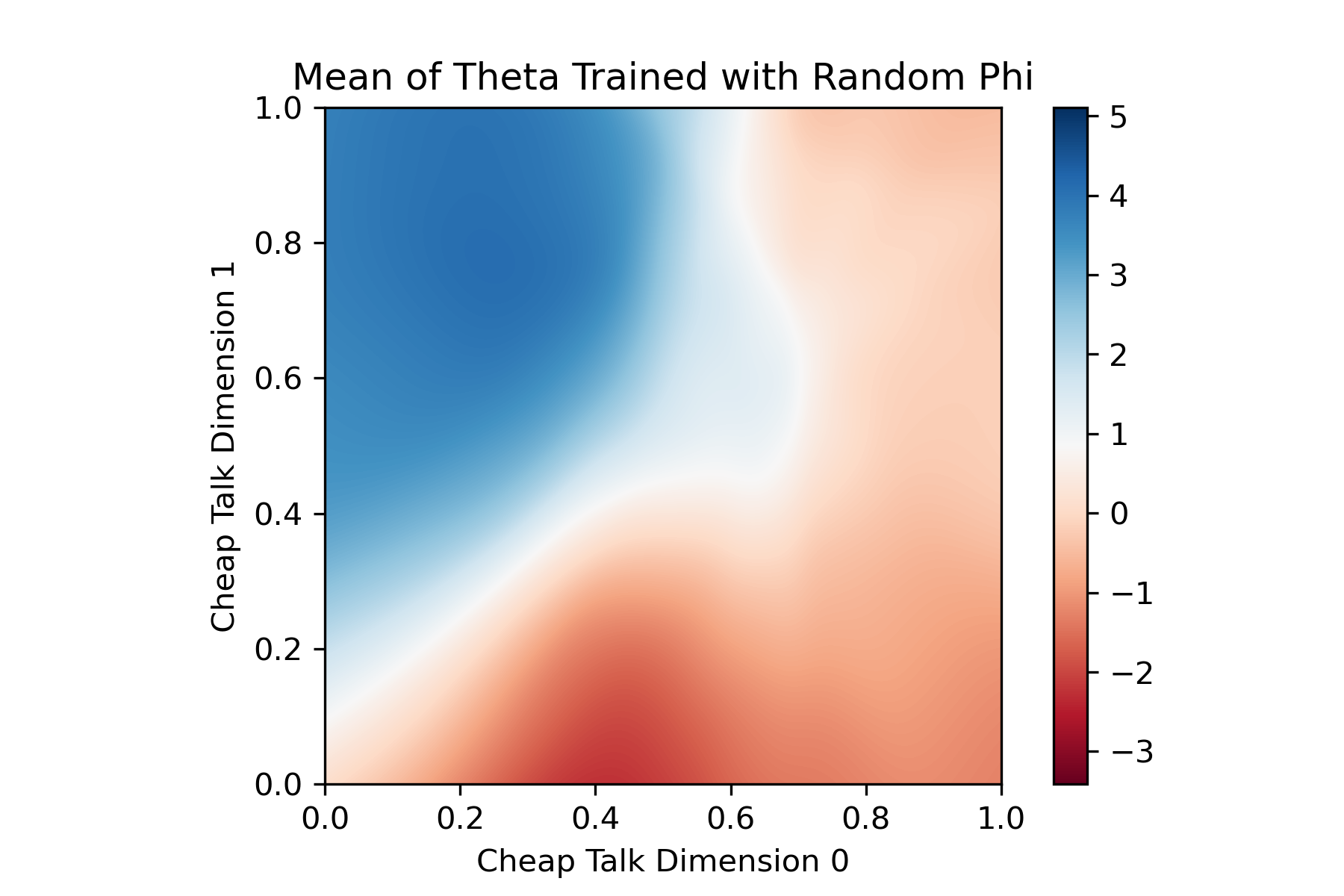}
	\caption{}
	\label{fig:mean_random_0}
 \end{subfigure}
  \begin{subfigure}[]{0.35\linewidth}
     \centering
	\includegraphics[width=\linewidth]{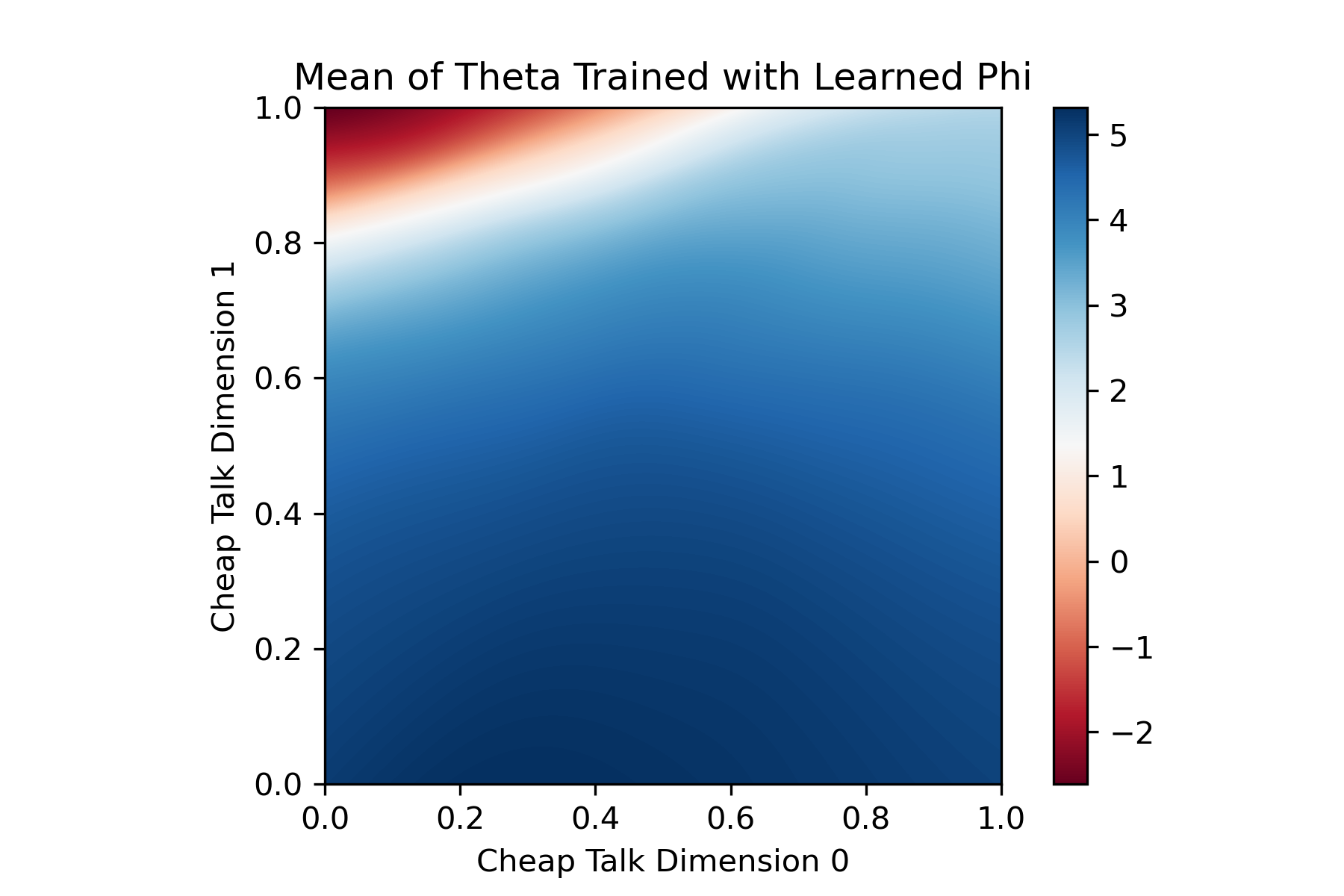}
	\caption{}
	\label{fig:mean_learned_1}
 \end{subfigure}
   \begin{subfigure}[]{0.35\linewidth}
     \centering
	\includegraphics[width=\linewidth]{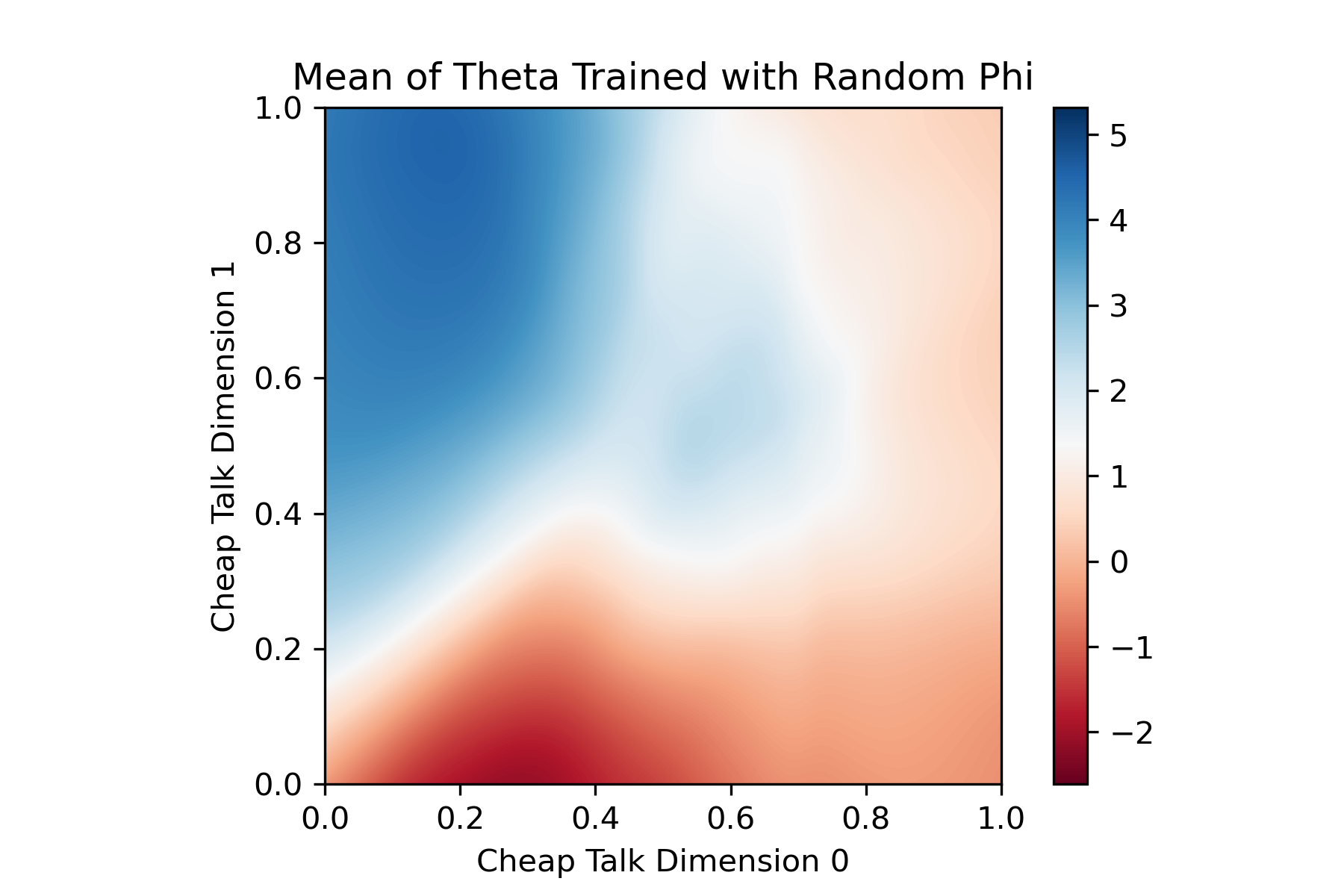}
	\caption{}
	\label{fig:mean_random_1}
 \end{subfigure}
  \begin{subfigure}[]{0.35\linewidth}
     \centering
	\includegraphics[width=\linewidth]{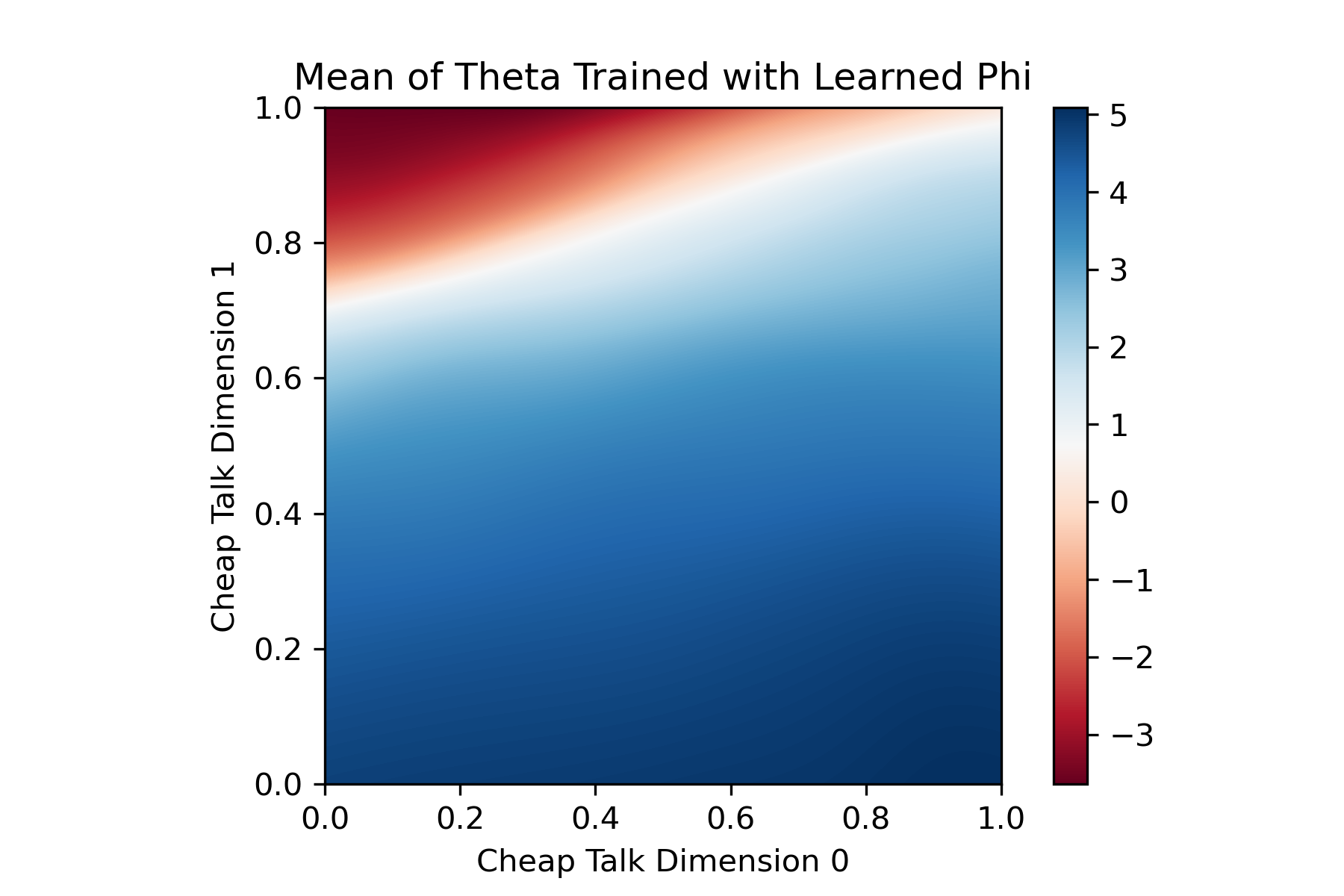}
	\caption{}
	\label{fig:mean_learned_2}
 \end{subfigure}
  \begin{subfigure}[]{0.35\linewidth}
     \centering
	\includegraphics[width=\linewidth]{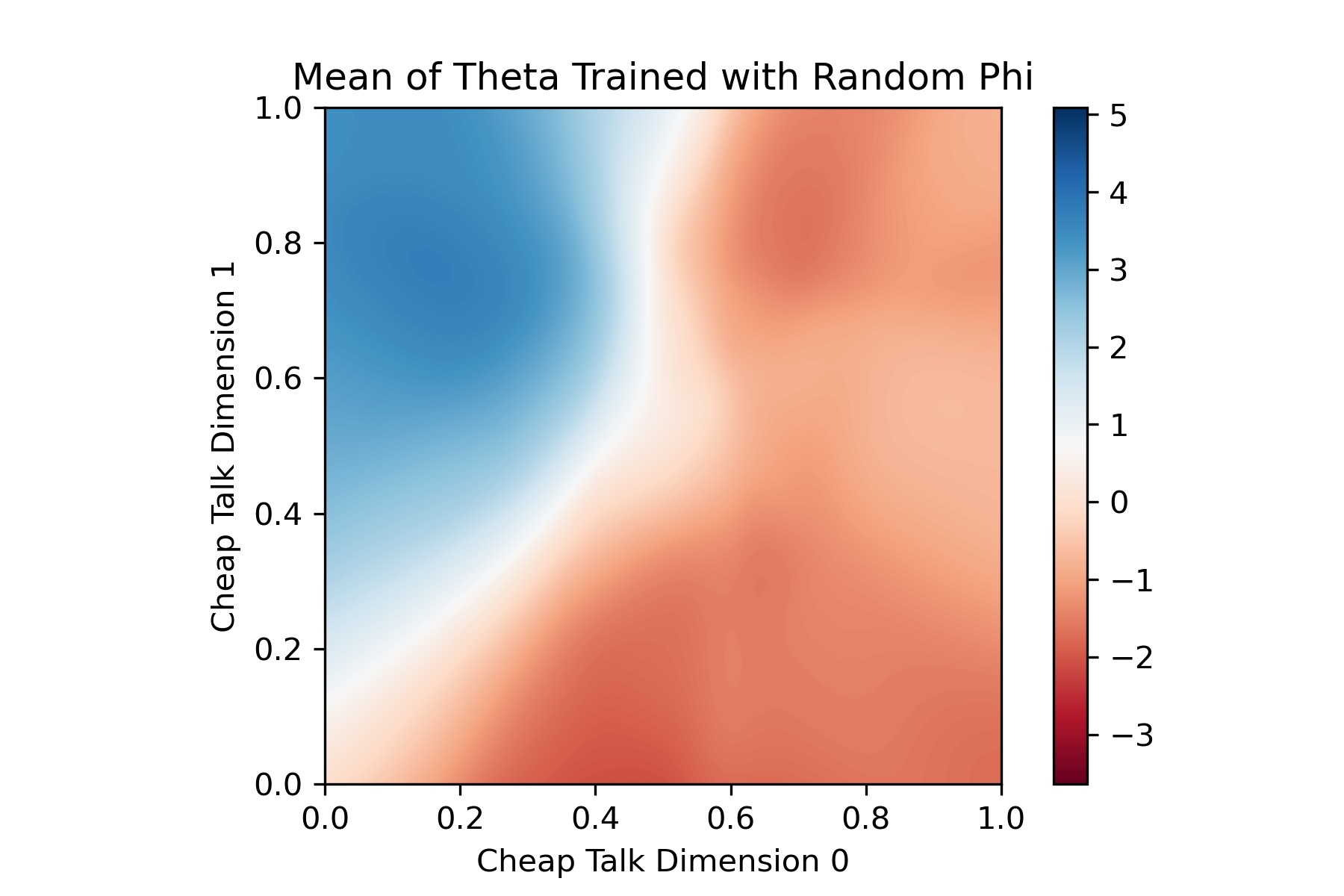}
	\caption{}
	\label{fig:mean_random_2}
 \end{subfigure}
  \begin{subfigure}[]{0.35\linewidth}
     \centering
	\includegraphics[width=\linewidth]{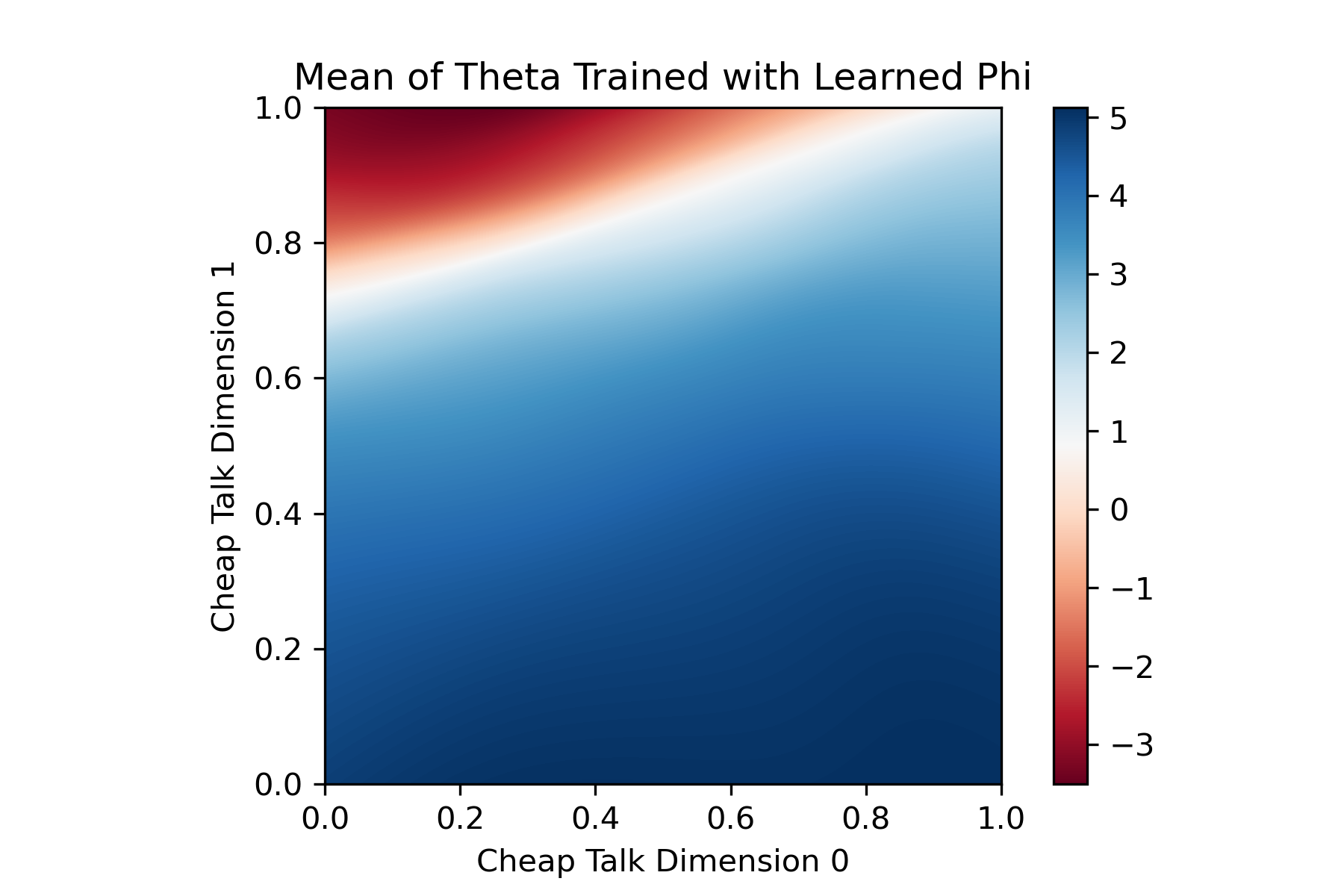}
	\caption{}
	\label{fig:mean_learned_3}
 \end{subfigure}
   \begin{subfigure}[]{0.35\linewidth}
     \centering
	\includegraphics[width=\linewidth]{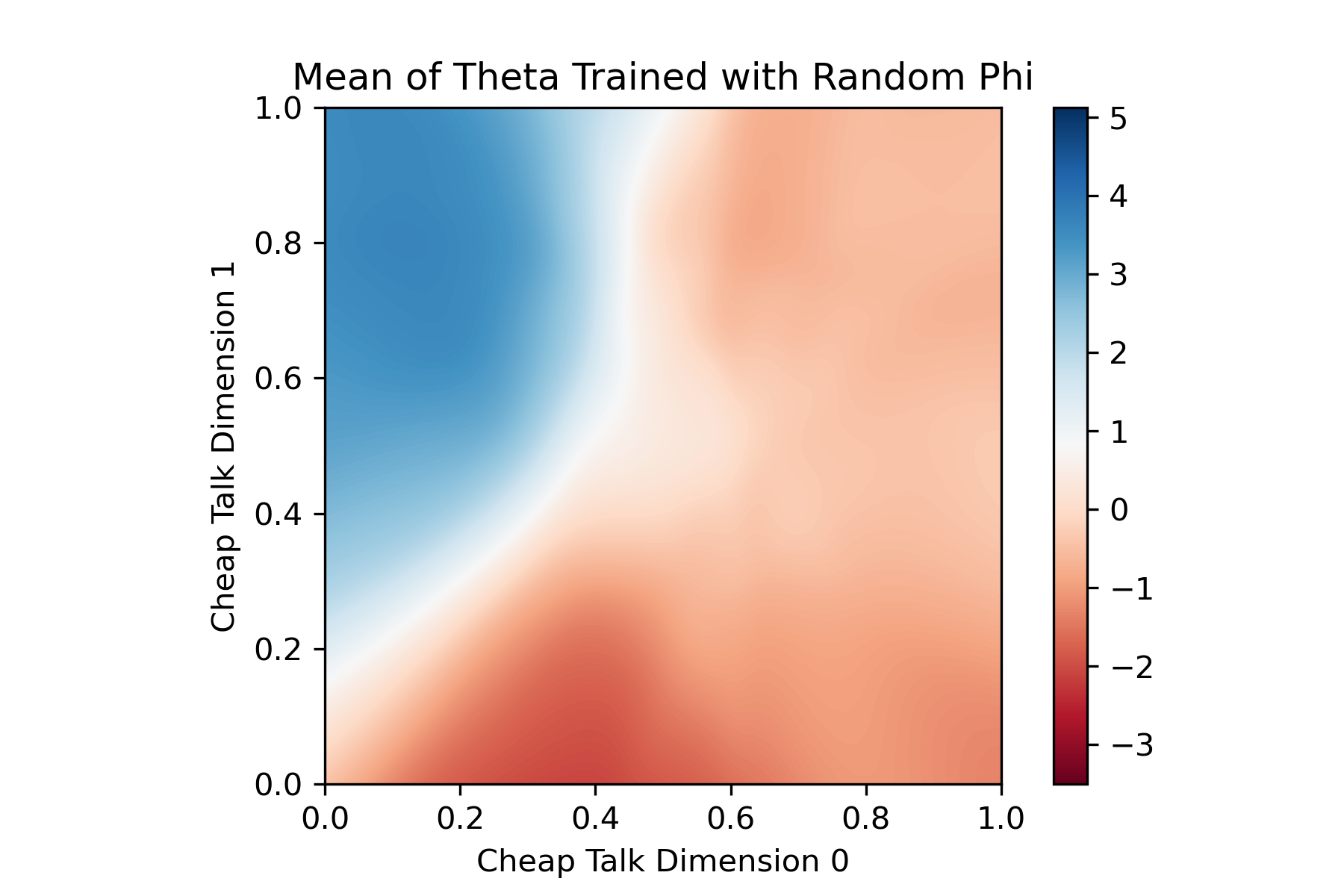}
	\caption{}
	\label{fig:mean_random_3}
 \end{subfigure}
    \caption{We train 10 different \learner s\ alongside the Learned $\phi$ (left column), and 10 different \learner s\ alongside a randomly generated $\phi$ (right column) in the Pendulum environment. We show the mean of the policy output across the 10 \learner s\ as we vary the value of the \fname\, in multiple randomly selected states. The learned $\phi$ consistently generates similar policy outputs across different states with respect to the cheap talk channel, implying that the learned $\phi$ shapes the \learner\ in a consistent way.
    }
    \label{fig:Pendulum_Vis_More}
\vspace{-15pt}
\end{figure*}

\pagebreak

\section{Additive Perturbations}
\label{app:additive-attack}

\begin{figure*}[hb]
 \centering
 \begin{subfigure}[]{0.32\linewidth}
     \centering
	\includegraphics[width=\linewidth]{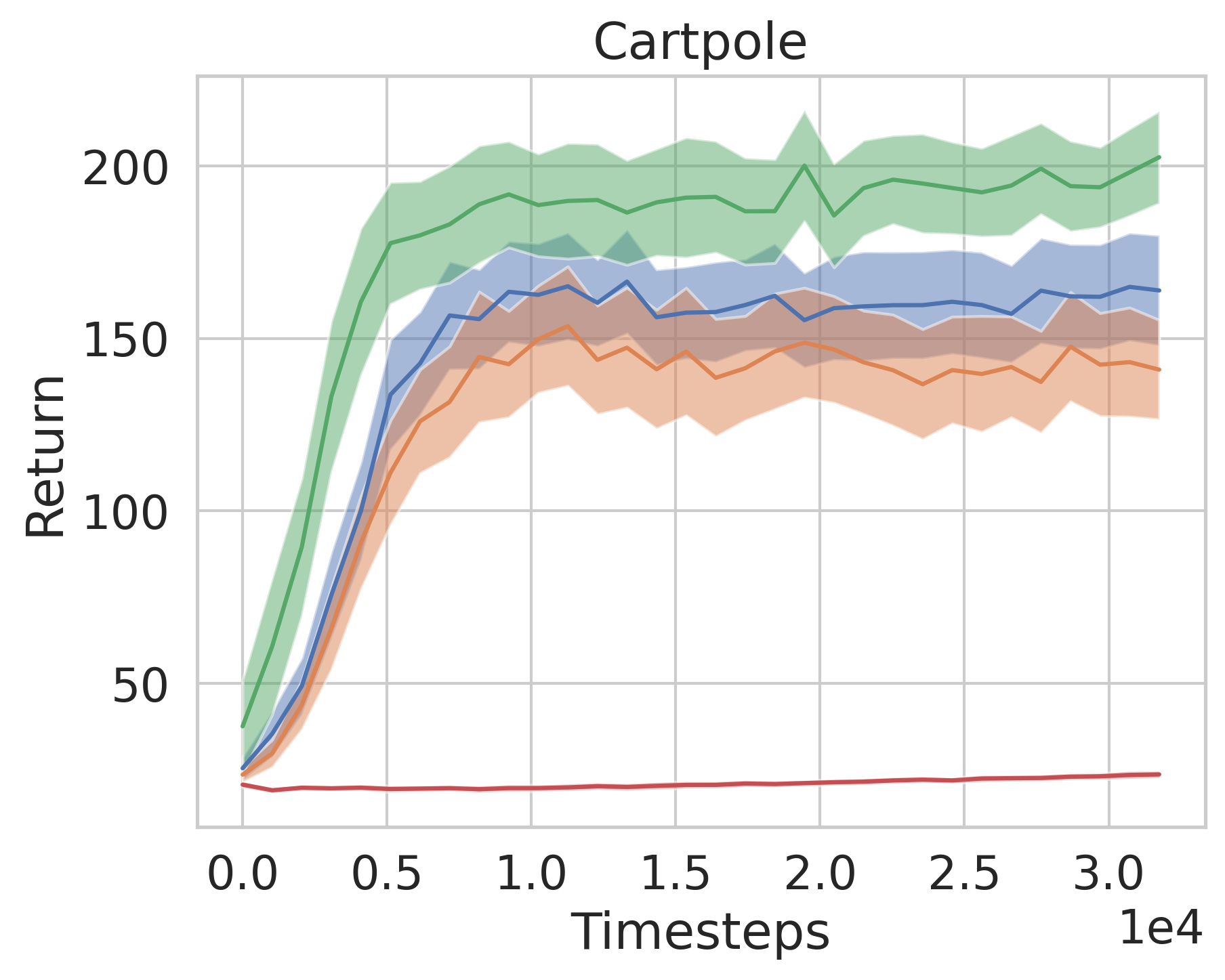}
	\caption{}
	\label{fig:cartpole_train_add}
 \end{subfigure}
  \begin{subfigure}[]{0.32\linewidth}
     \centering
	\includegraphics[width=\linewidth]{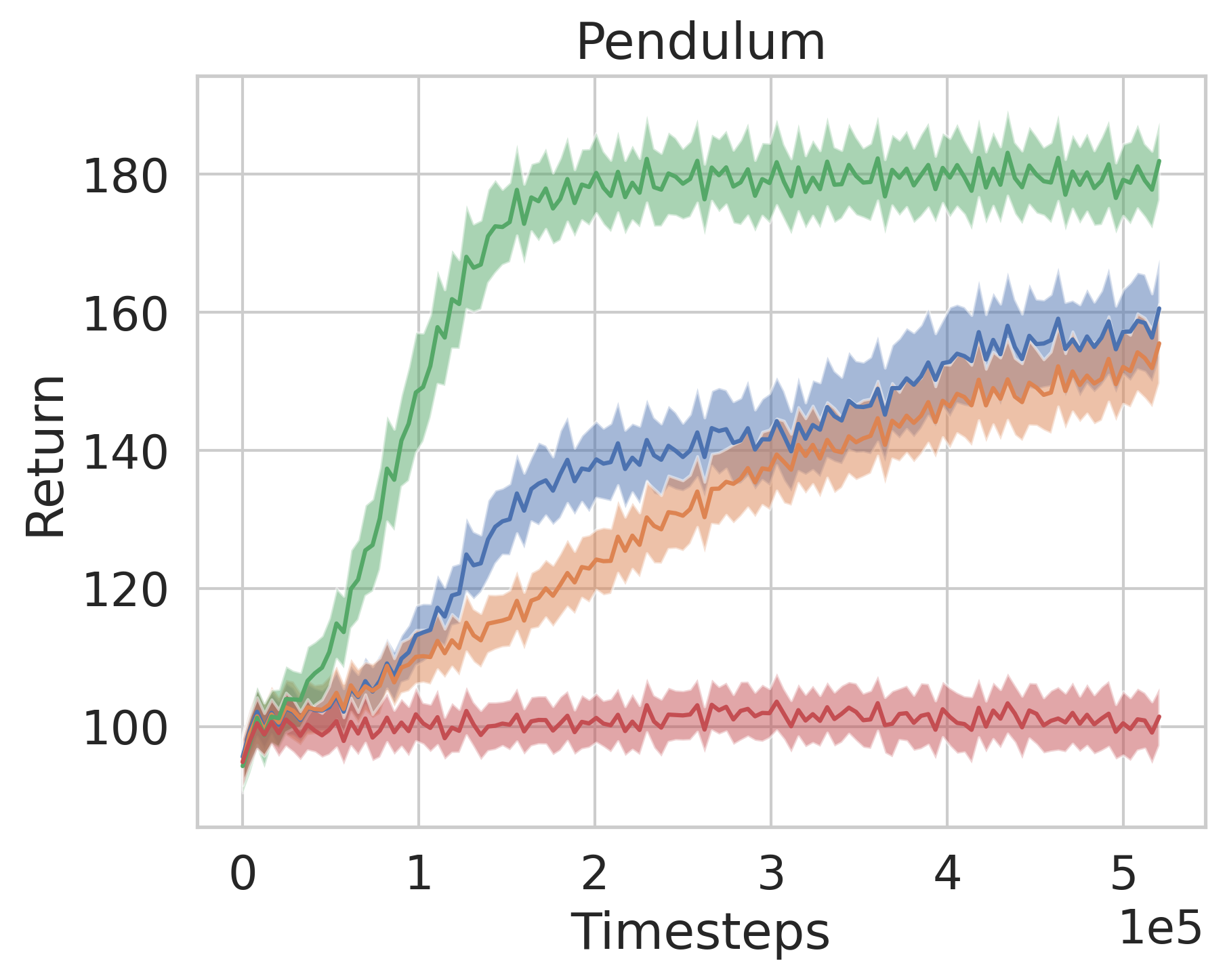}
	\caption{}
	\label{fig:pendulum_train_add}
 \end{subfigure}
  \begin{subfigure}[]{0.32\linewidth}
     \centering
	\includegraphics[width=\linewidth]{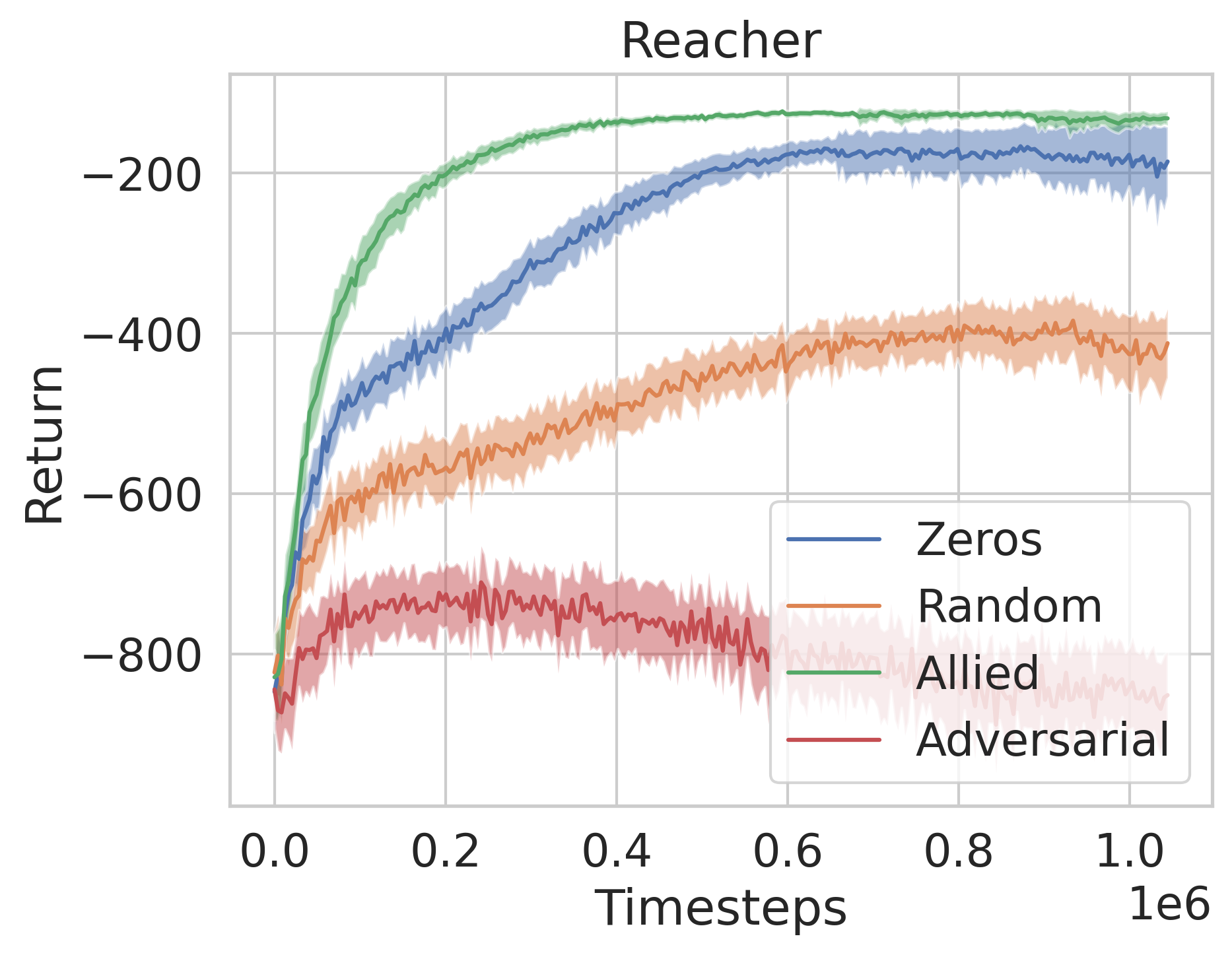}
	\caption{}
	\label{fig:reacher_train_add}
 \end{subfigure}
\caption{Visualisations of the training curves of the \learner\ across different Adversaries for (a) Cartpole, (b) Pendulum, and (c) Reacher. Error bars denote the standard error across $10$ seeds of Victims trained against a single trained Adversary. In this setting, the \shaper\ \textit{adds} the perturbation to the input rather than appending. Note that this allows the \shaper\ to conflate states and influence the optimal policy. Thus, the Adversary can harm performance far more.}
\label{fig:additive-ps}
\end{figure*}

\begin{figure*}[hb]
 \centering
 \begin{subfigure}[]{0.32\linewidth}
     \centering
	\includegraphics[width=\linewidth]{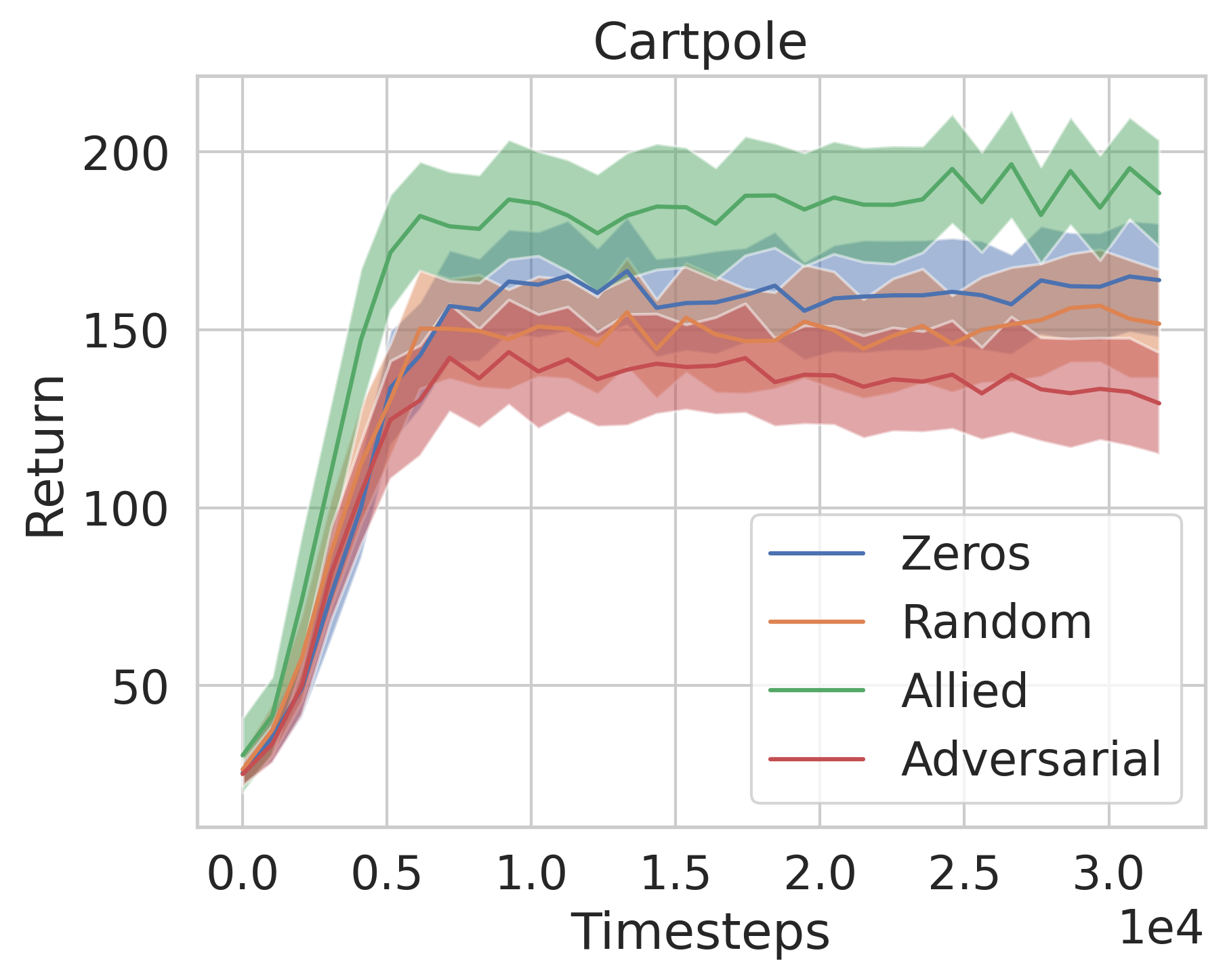}
	\caption{}
	\label{fig:cartpole_train_useless}
 \end{subfigure}
\caption{Visualisations of the training curves of the \learner\ across different Adversaries for Cartpole. Error bars denote the standard error across $10$ seeds of Victims trained against a single trained Adversary. In this setting, the \shaper\ \textit{adds} the perturbation to the useless features identified in \citep{lange2022lottery} rather than appending. It achieves similar performance to the cheap talk channel attacks.}
\label{fig:additive-useless-ps}
\end{figure*}

\pagebreak

% \section{Other RL Algorithms}

% \begin{figure*}[h]
%  \centering
% \includegraphics[width=\linewidth]{Figures/cartpole_dqn_preliminary_curves.png}
% \caption{}
%  \label{fig:dqn_preliminary}
% \end{figure*}

% \pagebreak

\section{Frequently Asked Questions}
\label{app:faq}

\textbf{Q: How does the Adversary know which channels don't influence the underlying environment?}

The Adversary does not need to know which channels don't influence the environment. The attacks in this paper are effective \textit{regardless} of whether or not the features are useless. The cheap talk channels are intended to represent the \textit{worst-case scenario} for the Adversary. Indeed, if the Adversary is able to influence important and useful channels, it is able to execute far more effective attacks. We show this in Appendix \ref{app:additive-attack} Figure \ref{fig:additive-ps}, where the Adversary, instead of appending to the observation, adds to it. 

\textbf{Q: Why would the Victim allow the Adversary to append to its observation? Isn't that easily detectable?}

The idea is not that a system designer would add a cheap-talk channel to their system but rather that such cheap-talk channels naturally occur in many real-world scenarios. We mention some real-world situations in the introduction (e.g. the recommender systems or financial models). However, practitioners currently also train with useless features in existing RL environments. Interestingly, \citet{lange2022lottery} has found that even the simple CartPole environment only requires two of the four features to learn the optimal policy.

We simulate these settings by appending to the observation; however, we could just as easily perturb the “useless” observations and expect similar results. We perturb the aforementioned useless features in CartPole and show the results in Appendix \ref{app:additive-attack} Figure \ref{fig:additive-useless-ps}.

\textbf{Q: What is the incentive for the attacker to be stationary?}

In real-world attacks, it would be far easier to implement a stationary adversary since it would just be a static function of the rest of the state. To implement a non-stationary attack, the adversary would need to know at what stage in training the victim is and at what point the training started.

Furthermore, we think the stationary adversarial attack is more relevant for studying curious failure modes in RL in general. This is a common perspective taken in adversarial attacks in machine learning since most popular attacks are impractical.

\end{document}